\documentclass[letterpaper]{article} 
\usepackage[preprint]{aaai2027}  
\usepackage[hyphens]{url}  
\usepackage{graphicx} 
\usepackage{natbib}  
\usepackage{caption} 
\usepackage{algorithm}
\usepackage{algorithmic}
\usepackage{amsmath,amssymb,amsfonts}
\usepackage{multirow}
\usepackage{makecell}
\usepackage{subcaption}
\usepackage{pdfpages}
\usepackage{pifont}

\usepackage{newfloat}
\usepackage{listings}
\DeclareCaptionStyle{ruled}{labelfont=normalfont,labelsep=colon,strut=off} 
\floatstyle{ruled}
\newfloat{listing}{tb}{lst}{}
\floatname{listing}{Listing}

\usepackage{booktabs}

\newtheorem{theorem}{Theorem}

\newenvironment{proof}{{\it Proof}.\ }{\hfill $\blacksquare$\par}

\definecolor{promptcolor}{HTML}{BFC1C0}
\usepackage[most]{tcolorbox}
\tcbuselibrary{skins, breakable}
\newtcolorbox{promptbox}[1]{
    breakable,
    colback=white,
    colframe=promptcolor,
    arc=4pt,
    boxrule=0.8pt,
    title=#1,
    coltitle=white,
    colbacktitle=promptcolor,
    fonttitle=\bfseries
}

\title{Learning What to Activate: Combinatorial Capability Allocation for Long-Horizon Multimodal Agents}

\author{
    Wenhao Yuan\textsuperscript{\rm 1}, Chenchen Lin\textsuperscript{\rm 1}, Jian Chen\textsuperscript{\rm 1}, Jinfeng Xu\textsuperscript{\rm 1}, Shuo Yang\textsuperscript{\rm 1}, Edith Cheuk-Han Ngai\textsuperscript{\rm 1}\corresponding
}

\affiliations{
    \textsuperscript{\rm 1}Department of Electrical and Computer Engineering, The University of Hong Kong

}

\begin{document}

\maketitle

\begin{abstract}
Long-horizon multimodal agents rely on specialized capabilities for perception, retrieval, reasoning, verification, and execution. Existing designs typically activate a fixed capability set or invoke a predefined workflow, incurring substantial computational overhead while failing to accommodate stage-dependent capability demands. In this paper, we study the \textit{combinatorial capability allocation} problem for long-horizon multimodal agent systems, where the system selects a cost-sensitive subset of specialized capabilities at each interaction stage, which is nontrivial since capability values depend on the selected subset, while previous allocations alter the states encountered by subsequent decisions. We introduce \textsc{CoCA}, an on-policy learning framework that recovers a deployable capability-subset policy from sparse conditional comparisons. On states visited by the student policy, the stronger teacher compares the marginal net values of candidate capabilities, conditioned on the currently selected subset. Then, we adopt a conditional utility model to transform such comparisons into an autoregressive capability-subset policy, avoiding explicit enumeration. We further introduce dual-level on-policy distillation to address distribution mismatch both across environment states and within the partial subsets encountered during set construction. Finally, trajectory-level reinforcement learning refines the distilled policy toward task success, activation cost, and allocation stability. At inference time, allocation is performed solely by the lightweight student policy without teacher queries or online updates. Experiments on long-horizon multimodal environments and controlled capability-demand shifts demonstrate the superiority of our method over the state-of-the-art baseline methods.
\end{abstract}


\section{Introduction}
Large Language Models (LLMs) are increasingly deployed as long-horizon multimodal agents that perceive, retrieve, reason, verify, and act over extended interactions~\citep{koh2024visualwebarena, tian2025mmina, xu2026evolution}. To support such heterogeneous demands, these systems maintain a library of specialized capabilities, each instantiated as a dedicated agent or tool and activated as tasks evolve~\citep{wu2024autogen, zhang2026evoroute}. However, prevailing designs either use a fixed capability set or follow a predefined workflow, overlooking that capability demands are \textit{stage-dependent} and shift with new user turns or visual observations~\citep{fan2025workflowllm, xu2025crab, zhang2026agentrouter}. We refer to this setting as \textit{combinatorial capability allocation}, where the system selects a cost-sensitive subset of capabilities at each interaction stage rather than committing to a static configuration. This problem is crucial in long-horizon settings since each allocation affects the observations and intermediate results to subsequent stages~\citep{xu2025theagentcompany}. An over-provisioned or misaligned selection not only incurs redundant activation costs but also propagates to future states, reshapes downstream capability requirements, and induces unstable, oscillatory allocation behavior across trajectories~\citep{ong2025routellm}.

In multimodal agent systems, existing methods primarily allocate capabilities through fixed capability sets~\citep{li2026organizing}, predefined workflows~\citep{fan2025workflowllm}, or model and tool routing~\citep{ong2025routellm, zhang2026mtrouter}. While simplifying system design, these methods typically assume fixed capability values. In practice, however, capability values are \textit{coupled}: within a stage, a capability's marginal net value depends on the co-activated capabilities; across stages, it depends on how prior allocations have reshaped the current state. Recent work has explored dynamic tool selection~\citep{liu2025toolace}, cost-aware routing~\citep{zhang2026mtrouter}, and adaptive agent composition~\citep{zhang2025agentorchestra}, yet such \textit{subset-valued} decisions remain underexplored~\citep{zheng2026skillselect}. This coupling also limits supervision: marginal net values vary with the selected subset and the policy-shaped state, while enumerating or labeling optimal subsets over the combinatorial action space is prohibitively costly~\citep{niu2026routing}. Existing strategies, such as imitating fixed expert pipelines or annotating stage-level subsets~\citep{xi2025agentgym}, therefore provide limited support for learning a deployable policy. Instead, supervision is more naturally expressed as \textit{relative}, condition-dependent preferences between candidate capabilities given the selected subset. This motivates recovering a subset policy from conditional comparisons while accounting for both intra-stage capability dependencies and inter-stage state coupling, raising a key question: \textit{How can a long-horizon multimodal agent decide which capabilities to activate at each stage under coupled capability values?}

To overcome these limitations, we propose \textbf{Co}nditional \textbf{C}apability \textbf{A}llocation (\textsc{CoCA}) Learning, an on-policy learning framework that recovers a deployable capability-subset policy from sparse conditional comparisons. To model intra-stage coupling, \textsc{CoCA} adopts a conditional pairwise utility model, where a stronger teacher compares the marginal net values of candidate capabilities given the currently selected subset. The learned utilities are then transformed into an autoregressive subset policy, avoiding explicit enumeration of the combinatorial action space. To address inter-stage coupling, \textsc{CoCA} introduces dual-level on-policy distillation that aligns the student with its deployment distribution over both environment states and partially constructed subsets through an annealed teacher--student mixture. Finally, trajectory-level reinforcement learning refines the distilled policy by propagating delayed rewards for task success, activation cost, and allocation stability to each include-or-skip decision, yielding a cost-effective and temporally stable policy. At inference time, allocation is performed solely by the lightweight student policy, without teacher queries or online updates. Our key contributions are summarized as follows:
\begin{itemize}
\item We identify the problem of combinatorial capability allocation in long-horizon multimodal agents, where subset-dependent capability values and policy-induced state shifts render static or single-choice activation inadequate.

\item We introduce \textsc{CoCA}, a novel framework that learns conditional pairwise utility, distills it into an autoregressive subset policy under dual-level distribution matching, and refines it via trajectory-level reinforcement learning.

\item We conduct extensive experiments on long-horizon multimodal environments, demonstrating the superiority of \textsc{CoCA} over state-of-the-art baselines.
\end{itemize}

\section{Related Work}

\subsection{Capability Allocation in Agentic Systems}
Multimodal agent systems increasingly assemble a library of specialized capabilities, each instantiated as a dedicated agent or tool, to cover the heterogeneous demands of long-horizon tasks~\citep{ruan2026aorchestra,yao2026ace}. Existing work allocates these capabilities through fixed candidate catalogs~\citep{wu2025agentic,zhao2025cola}, predefined workflows~\citep{wang2026always,wang2026fusionflow}, or model and tool routing that maps each query to a single capability~\citep{dekoninck2025unified,faghih2025tool}. To make allocation more responsive, recent methods explore dynamic tool selection~\citep{li2025adaptive,zhang2025toolexpnet}, cost-aware routing~\citep{ding2025best,xiang2026llm}, and adaptive agent composition or orchestration~\citep{liu2026evolving}. These works demonstrate that adapting which capabilities are invoked, rather than running a full pipeline, improves both efficiency and task performance~\citep{ding2025best,xiang2026llm}.

Despite these advances, most routing methods primarily model capability selection at the individual capability level, reducing allocation to factorized decisions~\citep{dekoninck2025unified, niu2026routing}. In long-horizon settings, however, capability values are \textit{coupled}: within a stage, marginal value depends on the co-activated capabilities~\citep{zhang2025toolexpnet}; across stages, prior allocations reshape the states governing subsequent decisions~\citep{jia2026autotool}. Although recent work has formulated agent routing as cost-aware set-valued prediction~\citep{nayan2026multi}, existing methods either select a single candidate or predict a subset in one shot, and thus cannot jointly capture intra-stage dependencies, inter-stage state coupling, and marginal values conditioned on partially selected subsets.

\subsection{On-Policy Learning from Preference Supervision}
A parallel line of work studies how to learn a deployable policy under distribution shift and weak supervision. On-policy imitation and distillation methods collect supervision on the states the student itself visits, mitigating the compounding errors that arise when a policy is trained only on expert trajectories~\citep{zhong2026sod,zhang2026fast}. To reduce the burden of dense action labels, preference-based learning instead elicits relative comparisons between candidates and fits a utility or reward model, most notably through Bradley--Terry style objectives~\citep{zhang2025beyond,cho2025policy}. These signals are then used to shape or refine policies, often combined with reinforcement learning to optimize delayed, trajectory-level outcomes~\citep{yang2026opid,wu2026seed}. Together, these works show that on-policy collection and relative preference supervision are effective when explicit optimal labels are unavailable or costly.

Despite recent progress, on-policy imitation typically assumes that the teacher provides a target action at each visited state. Here, however, the action is a capability \textit{subset}, and supervision consists only of sparse, condition-dependent pairwise comparisons, leaving unclear how to recover a well-defined subset policy. Preference-based methods further yield trajectory- or response-level scalar rewards rather than step-level distributions for dependency-aware set construction. Standard on-policy distillation also corrects only state-level shifts, overlooking mismatches in partially constructed subsets. Thus, existing methods cannot jointly recover an autoregressive subset policy from conditional comparisons and align the student across state and subset-prefix distributions.

\begin{figure*}[t]
\centering
\includegraphics[width=0.98\textwidth]{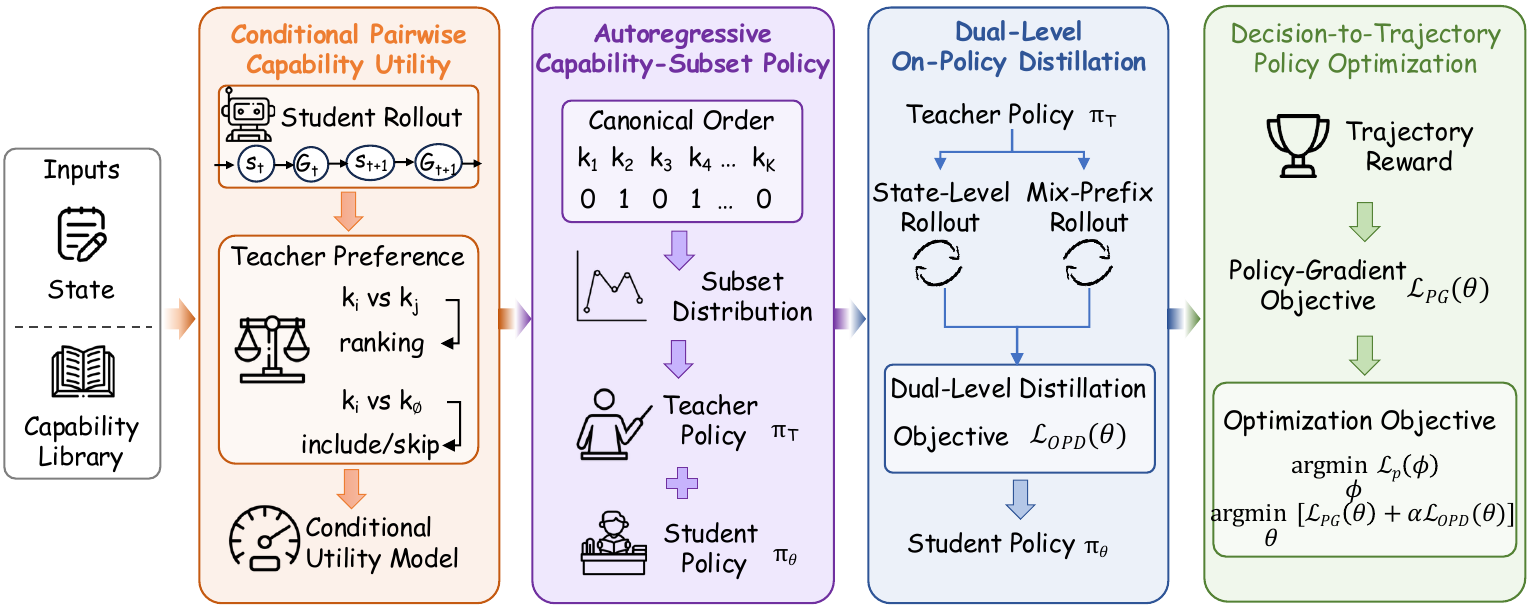}
\caption{An overview of the proposed \textsc{CoCA} framework. The diagram illustrates the learning pipeline from conditional preference modeling and teacher-guided policy construction to dual-level distillation and trajectory-level optimization.}
\label{framework}
\end{figure*}

\section{Methodology}

\subsection{Problem Formulation}
As illustrated in Figure~\ref{framework}, we consider a long-horizon multimodal agent system with a capability library $\mathcal{K}=\{k_1,\ldots, \\ k_K\}$, where each capability is instantiated by a specialized agent or tool providing a function. We divide task execution into interaction stages indexed by $t\in\{0,\ldots,T-1\}$. A new stage is triggered by information that may change capability requirements, such as a new user turn or visual observation, allowing capability allocation to adapt throughout execution.

At stage $t$, the system maintains a state $s_t=(x,z_t,G_{t-1})$, where $x$ denotes the task goal, $z_t$ the current context, and $G_{t-1}$ the previous capability set. Given $s_t$, the system selects an active subset $G_t\subseteq\mathcal{K}$, represented by a $K$-dimensional binary vector. The allocation decision determines capability selection, while execution follows a predefined dependency-aware order. The resulting capability outputs and observations update $z_{t+1}$ and form the next state $s_{t+1}$. Since $G_t$ affects future observations, capability allocation influences subsequent states and capability requirements.

Let $\pi_\theta(G_t\mid s_t)$ denote the allocation policy parameterized by $\theta$. For an execution trajectory $\xi=(s_0,G_0, \ldots,s_{T-1}, \\ G_{T-1},s_T)$, we define its cumulative activation cost as $\mathcal{C}_{A}(\xi)=\sum_{t=0}^{T-1}\sum_{k_i\in G_t}c_i$, where $c_i$ is a normalized capability-level cost proxy derived from inference overhead. Minimizing activation cost alone, however, may result in unstable allocation behavior, with capabilities being repeatedly activated and deactivated across adjacent stages. We further define the capability switch cost $\mathcal{C}_{S}(\xi)=\sum_{t=1}^{T-1}|G_t \triangle G_{t-1}|$, where $\triangle$ is the symmetric difference between two sets, and $|G_t\mathbin{\triangle}G_{t-1}|$ measures the number of capabilities whose activation status changes between consecutive stages. The objective is to maximize the task performance while controlling cost and allocation instability:
\begin{align}
\max_{\theta} \mathbb{E}_{\xi\sim\pi_\theta} \left[R_{\mathrm{task}}(\xi) -\lambda\mathcal{C}_{A}(\xi) -\mu\mathcal{C}_{S}(\xi) \right],
\label{eq:overall-objective}
\end{align}
where $R_{\mathrm{task}}(\xi)$ measures the task-level outcome of $\xi$, and $\lambda$ and $\mu$ control the trade-offs among task success, cost, and allocation stability. Optimizing this objective is challenging since existing benchmarks lack stage-level allocation labels, the policy itself shapes the future state distribution, and the subset-valued action creates dependencies among capabilities. These properties distinguish capability allocation from conventional static or single-choice routing.

\subsection{Conditional Pairwise Capability Utility} \label{sec:conditional-utility}
The capability allocation problem involves a combinatorial subset action space, where exhaustive supervision of optimal capability sets is costly. Moreover, capability utility is context-dependent, varying with the current state and previously selected capabilities. We therefore learn conditional capability utility from sparse pairwise comparisons rather than explicit subset labels.

For a state $s$ and an already-selected partial set $S\subseteq\mathcal{K}$, the teacher compares two candidate alternatives $a_i$ and $a_j$. We consider two types of comparisons: (\romannumeral1) $k_i \text{versus} k_j$; (\romannumeral2) $k_i \text{versus} k_{\varnothing}$, where $k_{\varnothing}$ denotes a null alternative corresponding to not activating the candidate capability. The first comparison captures relative preferences among capabilities, while the second determines whether adding $k_i$ provides positive marginal utility compared with leaving it unselected. To align teacher supervision with the deployment objective, we define capability preference according to cost-adjusted marginal utility. Specifically, the task-level marginal advantage and its cost-adjusted counterpart are defined as
\begin{align}
A_{k,\mathrm{task}}^\pi(s,S) &= Q_{\mathrm{task}}^\pi(s,S\cup\{k\}) - Q_{\mathrm{task}}^\pi(s,S), \\
U_{k,\lambda}^\pi(s,S) &= A_{k,\mathrm{task}}^\pi(s,S)-\lambda c_k,
\end{align}
with $U_{k_{\varnothing},\lambda}^\pi(s,S)=0$. The teacher evaluates candidates according to the cost-adjusted marginal utility criterion, and compares candidates based on the net-value criterion. Thus, the learned utility model $v_\phi$ captures the cost-adjusted marginal value of activating each capability. Each teacher query produces a preference label $y\in\{0,0.5,1\}$ and an optional weight $w\in(0,1]$. Specifically, $y=1$ indicates that $a_i$ is preferred, $y=0$ indicates that $a_j$ is preferred, and $y=0.5$ represents a tie. The resulting preference dataset is $\mathcal{D}_{\mathrm{p}} = \{(s,S,a_i,a_j,y,w)\}$. Then, we learn a conditional utility model $v_\phi(s,a\mid S)$ with a Bradley--Terry likelihood:
\begin{align}
P_\phi (a_i \succ a_j\mid s, S) \!=\! \sigma (v_\phi(s,a_i\mid S) - v_\phi(s,a_j\mid S)),
\end{align}
where $\sigma(\cdot)$ is the sigmoid function. Let $\Delta v_\phi = v_\phi(s,a_i\mid S)-v_\phi(s,a_j\mid S)$, the preference-learning objective is
\begin{align}
\!\!\! \mathcal{L}_{\mathrm{p}}(\phi) \!=\! -\mathbb{E}_{\mathcal{D}_{\mathrm{p}}} [w (y\log\sigma(\Delta v_\phi) \!+\! (1 \!-\! y)\log\sigma(\!-\! \Delta v_\phi))]. \!\!\! 
\label{preference_loss}
\end{align}

Conditioning on $S$ allows $v_\phi$ to capture capability dependencies by estimating context-dependent marginal utility. We anchor the null option with $v_\phi(s,k_{\varnothing}\mid S)\equiv0$ to resolve the utility shift ambiguity. Then, $v_\phi(s,k_i\mid S)>0$ directly indicates preference over skipping $k_i$, connecting pairwise utility estimation to autoregressive include-or-skip decisions. Since the target utility depends on $\pi$ through $Q_{\mathrm{task}}^\pi$, it evolves with the allocation policy, motivating preference collection on student-induced states and partial capability sets.

\subsection{Autoregressive Capability-Subset Policy} \label{sec:autoregressive-policy}
The learned utility scores individual capabilities, whereas the policy must generate a subset-valued action. We construct the subset distribution through a canonical autoregressive factorization. We assume a predefined dependency-aware execution protocol among capabilities and obtain an order $k_1\prec k_2\prec\cdots\prec k_K$ by linearizing this protocol. At stage $t$, the policy considers capabilities based on this order and makes a binary include-or-skip decision for each capability. Let $z_{t,i}\in\{0,1\}$ denote the decision for $k_i$, where $z_{t,i}=1$ indicates inclusion, and define the selected prefix before position $i$ as $S_{t,<i}=\{k_j\mid j<i,\;z_{t,j}=1\}$. The subset is $G_t=\{k_i\mid z_{t,i}=1\}$, whose probability is factorized as
\begin{align} \label{subset_factorization}
\pi(G_t\mid s_t)=\prod_{i=1}^{K}\pi(z_{t,i}\mid s_t,S_{t,<i}),
\end{align}
which assigns a unique probability to each subset under the canonical order. The policy makes exactly $K$ binary decisions without requiring additional stopping actions. Based on the conditional utility model $v_\phi$, we derive a utility-induced teacher distribution over each include-or-skip decision:
\begin{align} \label{teacher_policy}
\pi_T(z_{t,i}=1\mid s_t,S_{t,<i}) = \sigma(v_\phi(s_t,k_i\mid S_{t,<i})).
\end{align}

Substituting the teacher decision distribution into the autoregressive factorization in Eq.~\eqref{subset_factorization} yields the teacher subset policy. The teacher policy transforms sparse pairwise supervision into a dense conditional distribution over capability-selection decisions, which is then distilled into a lightweight student policy for deployment. The student follows the same canonical factorization and parameterizes the inclusion probability using the current state and selection history. Specifically, we encode the state and capability descriptions as $h_s=\operatorname{Enc}_{\mathrm{state}}(s_t)$, $e_i=\operatorname{Enc}_{\mathrm{cap}}(k_i)$. The selected prefix and previous-stage capability set are represented as $g_{t,<i} = \operatorname{Pool}(\{e_j\mid k_j\in S_{t,<i}\})$, $g_{\mathrm{prev}} = \operatorname{Pool}(\{e_j\mid k_j\in G_{t-1}\})$. Here, $g_{\mathrm{prev}}$ provides temporal context extracted from the previous-stage capability allocation already contained in $s_t$. The student inclusion probability is parameterized as
\begin{align}
\pi_\theta (z_{t,i}=1\mid s_t,S_{t,<i}) = \sigma (f_\theta(h_s,e_i,g_{t,<i},g_{\mathrm{prev}})),
\end{align}
where $g_{t,<i}$ captures within-stage capability dependencies. The resulting student subset distribution is
\begin{align}
\pi_\theta(G_t\mid s_t) = \prod_{i=1}^{K} \pi_\theta (z_{t,i}\mid s_t,S_{t,<i}).
\end{align}

Since the teacher and student share the same binary construction process, the teacher distribution can be directly distilled into the student policy at each capability decision.

\subsection{Dual-Level On-Policy Distillation} \label{sec:dual-opd}
The autoregressive capability allocation policy suffers from distribution shifts at both the state and prefix levels. At the state level, early allocation decisions alter future observations and intermediate results, causing student rollouts to visit states beyond teacher trajectories. At the prefix level, early include-or-skip decisions change the partial capability sets that condition later decisions. Therefore, teacher-only trajectories and prefixes cannot fully reflect the student deployment distribution. We address these shifts by collecting states from student rollouts $s_t\sim d^{\pi_\theta}$ and exposing the student to its own prefixes through a mixture rollout strategy. Specifically, we define $\pi_{\mathrm{mix}}^{(\eta)}(z_{t,i}\mid s_t,S_{t,<i})=\eta\pi_T(z_{t,i}\mid s_t,S_{t,<i})+(1-\eta)\pi_\theta(z_{t,i}\mid s_t,S_{t,<i})$, where $\eta$ is annealed from $1$ to $0$. Let $q_{\eta,i}(S_{t,<i}\mid s_t)$ denote the induced prefix distribution over the first $i-1$ decisions, and define the local distillation loss as $\ell_i(s_t,S_{t,<i})\!=\!D_{\mathrm{KL}}(\pi_T(z_{t,i}\mid s,S_{t,<i})\Vert\pi_\theta(z_{t,i}\mid s_t,S_{t,<i}))$. The distillation objective is defined as
\begin{align} \label{opd_loss}
\!\! \mathcal{L}_{\mathrm{D}}(\theta) \!=\! \mathbb{E}_{s_t\sim d^{\pi_\theta}}\!\! \left[\sum_{i=1}^{K} \mathbb{E}_{S_{t,<i}\sim q_{\eta,i}(\cdot\mid s_t)} [\ell_i(s_t,S_{t,<i})]\right]. \!\!
\end{align}

With $\eta$ close to $1$, the prefix distribution is dominated by teacher-generated selections, providing stable supervision. As $\eta$ decreases, the distribution gradually incorporates student-generated prefixes, reducing mismatch between training and deployment. When $\eta=1$, the prefix distribution is induced by the teacher policy ($q_{\eta, i}=d_i^{\pi_T}$), and the proposed objective reduces to the chain-rule decomposition of the forward subset-level KL divergence $D_{\mathrm{KL}} (\pi_T(G\mid s)\Vert\pi_\theta(G\mid s)) = \sum_{i=1}^{K} \mathbb{E}_{S_{t,<i} \sim d_i^{\pi_T}(\cdot\mid s)}[\ell_i(s_t,S_{t,<i})]$. For $\eta<1$, the objective replaces teacher-induced prefixes with mixed teacher--student prefixes, extending standard distillation to student-relevant decision contexts. State-level rollout and prefix-level mixing explicitly target distribution mismatch at both the trajectory and decision levels.

\subsection{Decision-to-Trajectory Policy Optimization} \label{joint_training}

While conditional distillation provides dense supervision for individual include-or-skip decisions, it does not directly optimize delayed task outcomes or trajectory-level allocation stability. We therefore further refine the student policy using the trajectory-level objective introduced in Eq.~\eqref{eq:overall-objective}. For a sampled trajectory $\xi$, the reward is defined as $R(\xi) = R_{\mathrm{task}}(\xi) - \lambda\mathcal{C}_{A}(\xi) - \mu\mathcal{C}_{S}(\xi)$. Using a variance-reduction baseline $b$, the policy-gradient objective is
\begin{align} \label{policy_gradient_loss}
\!\! \mathcal{L}_{\mathrm{PG}}(\theta) = -\mathbb{E}_{\xi\sim\pi_\theta} \left[(R(\xi)-b) \sum_{t=0}^{T-1} \log\pi_\theta(G_t|s_t) \right]. \!\!
\end{align}

Under the autoregressive capability-subset factorization, $\log\pi_\theta(G_t|s_t)=\sum_{i=1}^{K}\log\pi_\theta(z_{t,i}|s_t,S_{t,<i})$, allowing trajectory rewards to propagate to individual capability decisions. Activation cost serves different roles across learning stages: preference supervision provides a local cost-aware prior, while reinforcement learning optimizes the final task--cost trade-off. In contrast, switching cost depends on the entire allocation trajectory and is optimized only through the trajectory-level objective. The utility model and student policy are optimized with separate objectives:
\begin{align}
\phi^\star &= \arg\min_\phi\mathcal{L}_{\mathrm{p}}(\phi), \\
\theta^\star &= \arg\min_\theta \left[\mathcal{L}_{\mathrm{PG}}(\theta) + \alpha\mathcal{L}_{\mathrm{D}}(\theta)\right],
\end{align}
where $\alpha$ balances decision-level distillation and trajectory-level optimization. Specifically, $\mathcal{L}_{\mathrm{D}}$ provides capability-level supervision, whereas $\mathcal{L}_{\mathrm{PG}}$ optimizes delayed trajectory outcomes. Training follows an iterative on-policy procedure, where student trajectories update the preference model and teacher guidance for policy refinement. The utility model and student policy are optimized with separate schedules to stabilize their interaction under the evolving student-induced distribution. During inference, only the lightweight student policy $\pi_\theta$ performs autoregressive capability selection.


\section{Theoretical Analysis} \label{sec:theory}
In this section, we characterize error propagation from conditional utility learning to trajectory-level optimization. We analyze how sparse preference supervision affects teacher policy stability, how autoregressive distillation transfers policies under deployment prefixes, and how the resulting discrepancies influence long-horizon performance under distribution shifts. Proofs are provided in the supplementary material.

We characterize how conditional utility estimation errors affect the induced teacher policy. For a condition $c \!=\! (s,S)$, let $\mathcal{A}_c=(\mathcal{K}\setminus S)\cup\{k_{\varnothing}\}$ be the candidate set with the null option as reference. Fixing the continuation policy $\pi$ within an outer iteration, define the target utility as $u_c^\star(a):=U_{a,\lambda}^{\pi}(s,S)$ for non-null candidates and $u_c^\star(k_{\varnothing})=0$. Under Bradley--Terry preference consistency, $p_{ab}^\star(c)= \sigma(u_c^\star(a)- u_c^\star(b))$. Let $\widehat u_c$ be the learned utility vector, with residual $\rho_c=\|\nabla\mathcal{R}_c(\widehat u_c)\|_2$, $\mathcal{R}_c(u) = -\sum_{(a,b)\in\mathcal{E}_c} w_{ab}(c) [y_{ab}\log\sigma(u_a-u_b)+ (1-y_{ab})\log\sigma(u_b \\ -u_a)]$. The comparison structure is represented by the anchored Laplacian $L_c=\sum_{(a,b)\in\mathcal{E}_c}w_{ab}(c)b_{ab}b_{ab}^{\top}$, where $b_{ab}$ removes the null coordinate. Let $\gamma_c$ be the smallest eigenvalue of $L_c$, and assume that pairwise utility differences along the segment connecting $u_c^\star$ and $\widehat u_c$ are bounded as $|u_a-u_b|\leq B$, yielding $\omega_B=\sigma(B)(1-\sigma(B))>0$.
\begin{theorem} \label{thm_preference_subset_stability} 
For every condition $c$ satisfying the above assumptions, the utility estimation error satisfies $\|\widehat u_c-u_c^\star\|_2 \leq \frac{\rho_c}{\omega_B\gamma_c}$. Let $\pi^\star$ denote the oracle autoregressive subset policy induced by $u^\star$, and let $\pi_T$ denote the teacher policy induced by $\widehat u$. Define $\Delta_c(s,S)=\rho_{(s,S)}/(\omega_B\gamma_{(s,S)})$ and $\mathbb{E}^{\pi_T}_i[\cdot]:=\mathbb{E}_{S_{t,<i}\sim d_i^{\pi_T}(\cdot\mid s_t)}[\cdot]$. Then,
\begin{align} \label{eq:theory-preference-subset-bound}
\!\! D_{\mathrm{TV}} (\pi^\star(\cdot\mid s),\pi_T(\cdot\mid s)) \!\leq\! \frac{1}{4} \sum_{i=1}^{K} \mathbb{E}^{\pi_T}_i [\Delta_c(s_t,S_{t,<i})]. \!
\end{align}
\end{theorem}

Theorem~\ref{thm_preference_subset_stability} establishes the connection between sparse pairwise supervision and the reliability of the induced teacher policy, showing that teacher deviation is determined by $\rho_c$ and $\gamma_c$: accurate preference fitting alone is insufficient when the comparison structure is poorly conditioned. We next analyze teacher--student distillation under student-induced prefixes. For a state $s$, let $P_i^s(\cdot\mid S_{t,<i})=\pi_T(\cdot\mid s,S_{t,<i})$ and $Q_i^s(\cdot\mid S_{t,<i})=\pi_\theta(\cdot\mid s,S_{t,<i})$ denote the teacher and student conditionals, with subset distributions $P^s=\pi_T(\cdot\mid s)$ and $Q^s=\pi_\theta(\cdot\mid s)$. The prefix-level distillation error is $\delta_i(s)=\mathbb{E}_{S_{t,<i}\sim d_i^{\pi_\theta}(\cdot\mid s)}[D_{\mathrm{KL}}(P_i^s(\cdot\mid S_{t,<i})\Vert Q_i^s(\cdot\mid S_{t,<i}))]$.

\begin{theorem} \label{thm_student_prefix_composition} 
For any state $s$, the discrepancy between the complete teacher and student subset distributions satisfies
\begin{align}
D_{\mathrm{TV}}\!\left(P^s,Q^s\right) \leq \sqrt{\frac{K}{2} \sum_{i=1}^{K}\delta_i(s)}.
\label{eq:theory-local-global-bound}
\end{align}
\end{theorem}

Theorem~\ref{thm_student_prefix_composition} justifies distillation on student-generated prefixes by showing that local decision errors on deployment-relevant partial sets control the complete subset discrepancy. We next analyze the resulting distribution mismatch between training and deployment. Let $\nu(s)$ and $q_i(S_{t,<i}\mid s)$ denote the training state and prefix distributions, and define the position-wise distillation error as $\epsilon_i=\mathbb{E}_{s\sim\nu}\mathbb{E}_{S_{t,<i}\sim q_i(\cdot\mid s)} [D_{\mathrm{KL}}(P_i^s(\cdot\mid S_{t,<i})\Vert  Q_i^s(\cdot\mid S_{t,<i}))]$. The deployment mismatch is characterized by the state and prefix ratios  $C_{\mathrm{e}}=\sup_s\frac{\bar d^{\pi_\theta}(s)}{\nu(s)}$ and $C_{\mathrm{s},i}=\sup_{s,S_{t,<i}}\frac{d_i^{\pi_\theta}(S_{t,<i}\mid s)}{q_i(S_{t,<i}\mid s)}$, where $\bar d^{\pi_\theta}(s)=\frac{1}{T}\sum_{t=0}^{T-1}d_t^{\pi_\theta}(s)$.
Let $J(\pi)=\mathbb{E}_{\xi\sim\pi}R(\xi)$ with advantage $A_t^\pi(s,G)=Q_t^\pi(s,G)-V_t^\pi(s)$.

\begin{theorem} \label{thm_nested_performance} 
Assume that $C_{\mathrm{e}}$ and $C_{\mathrm{s},i}$ are finite and that the teacher advantage under the trajectory objective is uniformly bounded $\left|A_t^{\pi_T}(s,G)\right| \leq \bar A_T$, for all $t,s,G$. Then the return difference between the student and teacher satisfies
\begin{align}
\!\! |J(\pi_\theta) \!-\! J(\pi_T)| \!\leq\! 2T\bar A_T \min \left\{1, \sum_{i=1}^{K} \sqrt{\frac{C_{\mathrm{e}}C_{\mathrm{s},i}\epsilon_i}{2}}\right\}. \!\!
\label{eq:theory-nested-performance}
\end{align}
\end{theorem}

Theorem~\ref{thm_nested_performance} characterizes how distribution shifts propagate to long-horizon performance through two sources of mismatch: state-level shift from off-policy trajectory collection and prefix-level shift from autoregressive capability selection. On-policy state rollout removes state mismatch, while prefix mixing improves coverage of student partial sets. The bound is relative to the teacher policy and may scale with the horizon through $\bar A_T$, providing an error-dependence characterization rather than a linear-horizon guarantee.

\begin{table*}[t]
\centering
\renewcommand\arraystretch{0.9}
\resizebox{1.0\textwidth}{!}{
\begin{tabular}
{c|c|c|c|cc|cc|cc|cc} 
\toprule[1.2pt]
\multirow{2}{*}{\multirowcell{1.5}{\centering\textbf{Models}}} & \multirow{2}{*}{\multirowcell{1.5}{\centering\textbf{Baselines}}} & \multirow{2}{*}{\multirowcell{1.5}{\centering\textbf{Multi-Agent}}} & \multirow{2}{*}{\multirowcell{-1.1}{\makecell{\textbf{Adaptive} \\ \textbf{Control}}}}  & \multicolumn{2}{c|}{\centering\textbf{OSWorld}} & \multicolumn{2}{c|}{\centering\textbf{VisualWebArena}} & \multicolumn{2}{c|}{\centering\textbf{GAIA}} & \multicolumn{2}{c}{\centering\textbf{MMMU-Pro}}  \\ \cmidrule[0.5pt](l{1pt}r{0pt}){5-12}

& & & & Succ. $\uparrow$ & Cost $\downarrow$ & Succ. $\uparrow$ & Cost $\downarrow$ & Acc. $\uparrow$ & Cost $\downarrow$ & Acc. $\uparrow$ & Cost $\downarrow$  \\ \cmidrule[0.8pt](l{1pt}r{0pt}){1-12}

\multirow{6}{*}{\multirowcell{2}{Qwen3.6-27B}}    
        & \textsc{ReAct} & \ding{56} & \ding{52} & 17.2 & 0.90 & 13.6 & 0.88 & 33.1 & 0.86 & 43.2 & 0.83   \\

        & \textsc{Toolformer} & \ding{56} & \ding{56} & 20.8 & 0.83 & 15.1 & 0.81 & 39.4 & 0.75 & 44.1 & 0.85   \\

        & \textsc{Puppeteer} & \ding{52} & \ding{56} & 24.9 & 0.72 & 19.0 & 0.73 & 38.7 & 0.70 & 47.6 & 0.71   \\

        & \textsc{AutoTool} & \ding{56} & \ding{56} & 26.3 & 0.64 & 22.4 & 0.59 & 44.0 & 0.62 & 46.9 & 0.66   \\

        & \textsc{NaviAgent} & \ding{56} & \ding{52} & 31.7 & 0.55 & 23.1 & 0.51 & 46.2 & 0.53 & 52.3 & 0.57   \\

        & \textsc{CoCA}(Ours) & \ding{52} & \ding{52} & \textbf{35.1} & \textbf{0.36} & \textbf{26.8} & \textbf{0.33} & \textbf{51.9} & \textbf{0.35} & \textbf{53.8} & \textbf{0.46}    \\ \midrule[0.8pt]

\multirow{6}{*}{\multirowcell{2}{Kimi-K2.6}}
       
        & \textsc{ReAct} & \ding{56} & \ding{52} & 23.5 & 0.87 & 20.2 & 0.85 & 41.6 & 0.83 & 51.4 & 0.79   \\

        & \textsc{Toolformer} & \ding{56} & \ding{56} & 28.1 & 0.79 & 21.7 & 0.77 & 47.2 & 0.72 & 54.7 & 0.81   \\

        & \textsc{Puppeteer} & \ding{52} & \ding{56} & 30.4 & 0.68 & 26.9 & 0.69 & 50.8 & 0.66 & 55.9 & 0.67   \\

        & \textsc{AutoTool} & \ding{56} & \ding{56} & 34.2 & 0.60 & 27.3 & 0.55 & 51.2 & 0.58 & 55.1 & 0.62   \\

        & \textsc{NaviAgent} & \ding{56} & \ding{52} & 36.8 & 0.51 & 32.5 & 0.47 & 57.4 & 0.49 & 60.5 & 0.53   \\ 

        & \textsc{CoCA}(Ours) & \ding{52} & \ding{52} & \textbf{43.3} & \textbf{0.31} & \textbf{35.6} & \textbf{0.29} & \textbf{60.9} & \textbf{0.31} & \textbf{62.5} & \textbf{0.42}    \\ 
        
\bottomrule[1.2pt]
\end{tabular}}
\caption{The overall evaluation results of \textsc{CoCA} and other baseline methods on four benchmark datasets. The best-performing method is marked by \textbf{bold}.}
\label{accuracy}
\end{table*}

\begin{table*}[t]
\centering
\renewcommand\arraystretch{1.0}
\resizebox{1.0\textwidth}{!}{
\begin{tabular}{c|ccccc|ccccc|ccccc} 
\toprule[1.2pt]
\multirow{2}{*}{\multirowcell{2}{\centering\textbf{Methods}}} & \multicolumn{5}{c|}{\centering\textbf{OSWorld}} & \multicolumn{5}{c|}{\centering\textbf{VisualWebArena}} & \multicolumn{5}{c}{\centering\textbf{GAIA}}  \\ \cmidrule[0.5pt](l{1pt}r{0pt}){2-16}

& SwC $\downarrow$ & DSA $\uparrow$ & PPA $\uparrow$ & TAR $\uparrow$ & UGC $\uparrow$ & SwC $\downarrow$ & DSA $\uparrow$ & PPA $\uparrow$ & TAR $\uparrow$ & UGC $\uparrow$ & SwC $\downarrow$ & DSA $\uparrow$ & PPA $\uparrow$ & TAR $\uparrow$ & UGC $\uparrow$  \\ \cmidrule[0.8pt](l{1pt}r{0pt}){1-16}

\textsc{ReAct} & 3.38 & 0.52 & 0.53 & 0.48 & 0.09 & 3.21 & 0.46 & 0.49 & 0.47 & 0.07 & 2.71 & 0.54 & 0.55 & 0.50 & 0.15 \\

\textsc{Toolformer} & 3.11 & 0.55 & 0.58 & 0.55 & 0.21 & 2.79 & 0.54 & 0.54 & 0.51 & 0.18 & 2.40 & 0.58 & 0.61 & 0.56 & 0.22 \\

\textsc{Puppeteer} & 2.43 & 0.62 & 0.61 & 0.59 & 0.29 & 2.36 & 0.57 & 0.62 & 0.60 & 0.30 & 1.88 & 0.65 & 0.64 & 0.62 & 0.38 \\

\textsc{AutoTool} & 2.17 & 0.64 & 0.69 & 0.66 & 0.42 & 1.98 & 0.63 & 0.65 & 0.61 & 0.35 & 1.69 & 0.66 & 0.72 & 0.69 & 0.44 \\

\textsc{NaviAgent} & 1.71 & 0.72 & 0.73 & 0.70 & 0.50 & 1.72 & 0.68 & 0.71 & 0.69 & 0.49 & 1.31 & 0.74 & 0.76 & 0.73 & 0.55 \\ \midrule[0.8pt]

\textsc{CoCA}(Ours) & \textbf{0.83} & \textbf{0.88} & \textbf{0.86} & \textbf{0.84} & \textbf{0.72} & \textbf{0.79} & \textbf{0.85} & \textbf{0.83} & \textbf{0.81} & \textbf{0.69} & \textbf{0.61} & \textbf{0.90} & \textbf{0.88} & \textbf{0.86} & \textbf{0.75} \\ 

\bottomrule[1.2pt]
\end{tabular}}
\caption{Evaluation of allocation quality under Kimi-K2.6. The best-performing method is marked by \textbf{bold}.}
\label{allocation}
\end{table*}

\begin{table*}[t]
\centering
\renewcommand\arraystretch{1.0}
\resizebox{1.0\textwidth}{!}{
\begin{tabular}{c|cccccc|cccccc} 
\toprule[1.2pt]
\multirow{2}{*}{\multirowcell{2}{\centering\textbf{Methods}}} & \multicolumn{6}{c|}{\centering\textbf{OSWorld}} & \multicolumn{6}{c}{\centering\textbf{GAIA}} \\ \cmidrule[0.5pt](l{1pt}r{0pt}){2-13}

& Succ. $\uparrow$ & SwC $\downarrow$ & DSA $\uparrow$ & PPA $\uparrow$ & TAR $\uparrow$ & UGC $\uparrow$ & Acc. $\uparrow$ & SwC $\downarrow$ & DSA $\uparrow$ & PPA $\uparrow$ & TAR $\uparrow$ & UGC $\uparrow$  \\ \cmidrule[0.8pt](l{1pt}r{0pt}){1-13}

w/o Set-Dep & 39.6 & 0.97 & 0.79 & 0.77 & 0.75 & 0.66 & 56.4 & 0.74 & 0.82 & 0.80 & 0.78 & 0.69  \\ 

w/o Set-On Policy & 41.2 & 0.91 & 0.84 & 0.78 & 0.76 & 0.68 & 58.1 & 0.69 & 0.86 & 0.81 & 0.79 & 0.71 \\ 

w/o Stage-On Policy & 38.4 & 0.95 & 0.81 & 0.82 & 0.80 & 0.67 & 55.7 & 0.72 & 0.84 & 0.85 & 0.83 & 0.70 \\ 

w/o Traj. RL & 40.1 & 1.58 & 0.86 & 0.85 & 0.83 & 0.69 & 57.3 & 1.19 & 0.89 & 0.87 & 0.85 & 0.72 \\ 

w/o Null-Anchor & 41.5 & 0.89 & 0.83 & 0.74 & 0.78 & 0.58 & 58.6 & 0.67 & 0.85 & 0.76 & 0.80 & 0.61 \\ \midrule[0.8pt]

\textsc{CoCA} & \textbf{43.3} & \textbf{0.83} & \textbf{0.88} & \textbf{0.86} & \textbf{0.84} & \textbf{0.72} & \textbf{60.9} & \textbf{0.61} & \textbf{0.90} & \textbf{0.88} & \textbf{0.86} & \textbf{0.75} \\ 

\bottomrule[1.2pt]
\end{tabular}}
\caption{Ablation study on the OSWorld and GAIA benchmarks under Kimi-K2.6.}
\label{ablation}
\end{table*}

\section{Experiments}

\subsection{Experimental Setup}

\subsubsection{Datasets}
We evaluate on four benchmarks that cover diverse aspects of agentic reasoning and multimodal task execution. \textbf{OSWorld}~\citep{xie2024osworld} evaluates computer-use agents in realistic operating system environments, requiring visual perception, tool interaction, and multi-step task completion. \textbf{VisualWebArena}~\citep{koh2024visualwebarena} assesses multimodal web agents through realistic website navigation and interaction tasks involving visual grounding and action planning. \textbf{GAIA}~\citep{mialon2024gaia} evaluates general AI assistants on real-world tasks requiring multi-step reasoning, external knowledge acquisition, and tool usage. \textbf{MMMU-Pro}~\citep{yue2024mmmu} tests advanced multimodal reasoning with expert-level questions requiring the integration of complex visual and textual information.

\subsubsection{Baselines}
We compare our method against state-of-the-art baselines. \textsc{ReAct}~\citep{yao2023react} enables language agents to interleave reasoning and action execution, allowing them to interact with external tools through iterative thought-action trajectories. \textsc{Toolformer}~\citep{schick2023toolformer} learns to invoke external tools by augmenting language models with self-supervised tool-use annotations. \textsc{Puppeteer}~\citep{dang2025multi} is a multi-agent tool-use framework that coordinates specialized agents and tools for complex task execution. \textsc{AutoTool}~\citep{jia2026autotool} automatically discovers and selects appropriate tools during inference to improve autonomous tool utilization. \textsc{NaviAgent}~\citep{jiang2026naviagent} is a navigation-oriented agent framework that dynamically plans and executes tool interactions for long-horizon tasks.

\subsubsection{Evaluation Metrics}
We adopt two complementary categories of metrics for evaluation. For \textit{Task-Level Performance}, we follow each benchmark's standard protocol, including execution-based success rate (Succ.) for OSWorld and VisualWebArena, quasi-exact-match accuracy (Acc.) for GAIA and MMMU-Pro, and the normalized inference cost per task, defined as a method's average total inference tokens (input and output across all activated capabilities and stages) divided by that of the \textit{all-open} baseline under the same executor, yielding a value in $[0,1]$ where lower is more efficient. For \textit{Allocation Quality}, since task success alone does not reflect whether capabilities are allocated efficiently and stably, we further evaluate allocation behavior. We report: (\expandafter{\romannumeral1}) \textit{Switching Cost} (SwC), the average per-stage symmetric difference between consecutive active capability sets, measuring allocation stability; (\expandafter{\romannumeral2}) \textit{Demand-Shift Adaptation Accuracy} (DSA), the fraction of stages where the selected capabilities match the ground-truth required capabilities under controlled demand shifts; (\expandafter{\romannumeral3}) \textit{Preference Prediction Accuracy} (PPA), the agreement between realized include/skip decisions and held-out oracle conditional pairwise preferences; (\expandafter{\romannumeral4}) \textit{Teacher-Agreement Rate} (TAR), the agreement between realized capability decisions and counterfactual rollout oracle decisions; and (\expandafter{\romannumeral5}) \textit{Utility--Gain Correlation} (UGC), the correlation between capability selection preference and the realized counterfactual gain of activation.

\subsubsection{Implementation Details}
We conduct experiments on Qwen3.6-27B~\citep{qwen36blog} and Kimi-K2.6~\citep{kimik26blog} as executors instantiating a capability library of size $K=8$, based on which we train the lightweight student allocator $\pi_\theta$. We use GPT-5.5~\citep{gpt55blog} as the teacher model. We set $\lambda=0.2$, $\mu=0.05$, and $\alpha=1.0$ following our sensitivity analysis. The utility model $v_\phi$ is trained with the Bradley--Terry objective using preference weight $w=1.0$, and the null anchor $v_\phi(s,k_{\varnothing}\mid S)\equiv 0$. The mixture weight $\eta$ is linearly annealed from $1$ to $0$ over the first half of training, and we use the batch-mean return as the baseline $b$ with within-batch advantage standardization. The utility model and student policy use separate AdamW optimizers with learning rates $1\text{e-}4$ and $5\text{e-}5$, trained over $5$ outer on-policy iterations. At inference, only $\pi_\theta$ performs autoregressive capability selection.

\subsection{Experimental Results}

\subsubsection{Overall Performance}
Table~\ref{accuracy} demonstrates that \textsc{CoCA} consistently achieves superior performance across benchmarks and backbone models while maintaining lower execution costs. This improvement can be attributed to \textsc{CoCA}’s collaborative orchestration mechanism, which dynamically assigns subtasks to appropriate agents and coordinates their execution trajectories instead of relying on a single-agent reasoning process or predefined tool invocation patterns. By jointly optimizing agent collaboration and execution planning, \textsc{CoCA} improves task completion reliability while avoiding unnecessary interactions. To further investigate the underlying mechanism, Table~\ref{allocation} evaluates the quality of agent allocation decisions during execution. \textsc{CoCA} achieves consistently better allocation quality across different environments, indicating that its gains originate from more accurate agent-task matching and adaptive coordination. In contrast, existing methods either lack explicit multi-agent collaboration or rely on fixed execution strategies, leading to suboptimal resource utilization. These results verify that effective orchestration, rather than simply increasing the number of agents or tool calls, is the key factor enabling reliable long-horizon task execution.

\begin{figure}[t]
\centering
\subfloat{
    \includegraphics[width=0.23\textwidth]{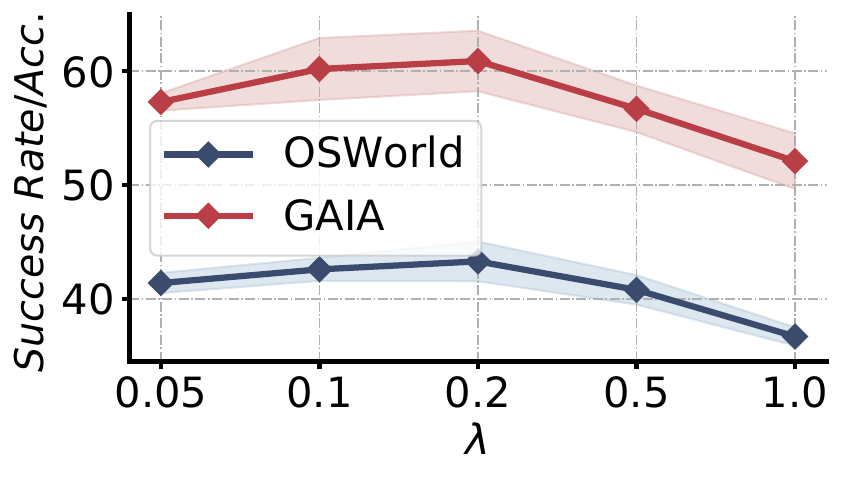}}
\subfloat{
    \includegraphics[width=0.23\textwidth]{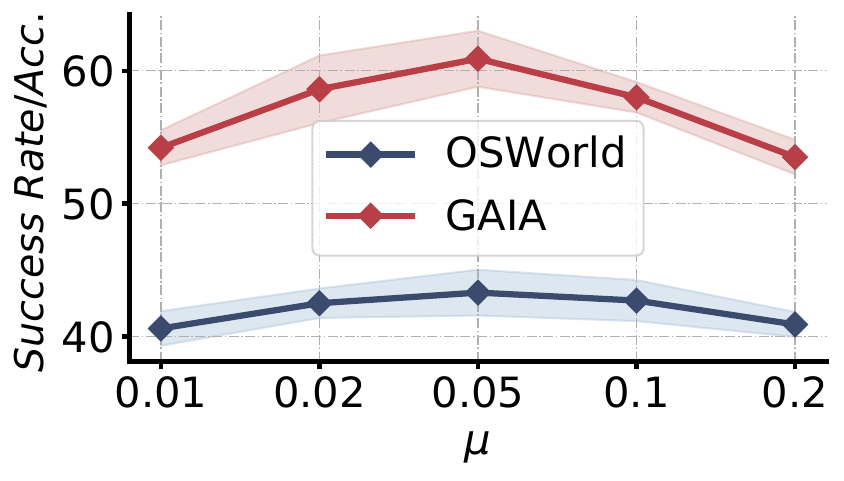}}
\vspace{5pt}
\subfloat{
    \includegraphics[width=0.23\textwidth]{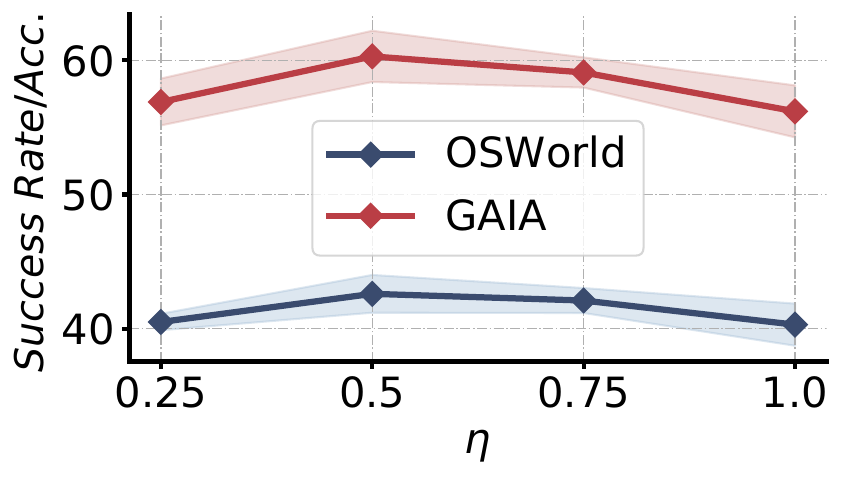}}
\subfloat{
    \includegraphics[width=0.23\textwidth]{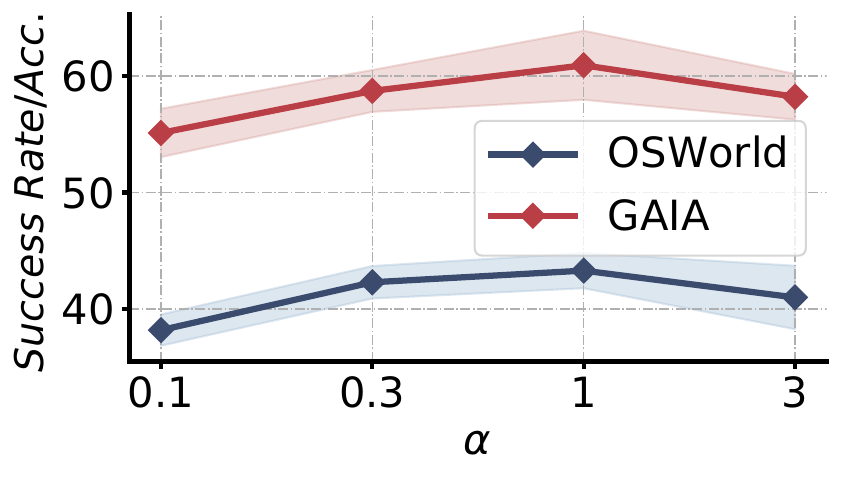}}
\caption{Hyperparameter sensitivity analysis on OSWorld and GAIA benchmarks under Kimi-K2.6.}
\label{hyper}
\end{figure}

\subsubsection{Ablation Studies}
In Table~\ref{ablation}, we evaluate the contribution of each component in \textsc{CoCA}. Removing any component consistently degrades performance, validating their complementary roles in long-horizon orchestration. Specifically, w/o Set-Dep shows the importance of dependency-aware decomposition for effective task planning, while w/o Set-On Policy and w/o Stage-On Policy demonstrate the necessity of adaptive decisions at both orchestration and execution levels. Removing Traj. RL leads to inferior execution quality and higher interaction cost, confirming the benefit of trajectory-level optimization. Meanwhile, w/o Null-Anchor mainly affects execution stability, highlighting the role of anchor-based regulation in preventing error accumulation.

\begin{table}[t]
\centering
\renewcommand\arraystretch{1.0}
\resizebox{0.48\textwidth}{!}{
\begin{tabular}{c|ccccc} 
\toprule[1.2pt]

\makecell{\centering\textbf{Methods}} & \makecell{Avg. \\ Latency} & \makecell{LLM \\ Calls} & \makecell{Token \\ Cost (k)} & \makecell{Avg. Act. \\ /stage}  & \makecell{Train \\ Cost}  \\   \cmidrule[0.5pt](l{1pt}r{0pt}){1-6}

\textsc{ReAct} & 62.4 & 14.8 & 41.2 & 6.7 & 0.0 \\

\textsc{Toolformer} & 55.1 & 12.3 & 35.7 & 5.9 & 1.2 \\

\textsc{Puppeteer} & 47.8 & 10.6 & 31.4 & 5.1 & 0.0 \\

\textsc{AutoTool} & 43.2 & 9.4  & 28.9 & 4.6 & 3.5 \\

\textsc{NaviAgent} & 38.5 & 8.1  & 24.3 & 3.8 & 5.8 \\ \midrule[0.8pt]

\textsc{CoCA}(Ours) & \textbf{31.7} & \textbf{6.5} & \textbf{18.6} & \textbf{2.9} & 7.4 \\

\bottomrule[1.2pt]
\end{tabular}}
\caption{Efficiency on GAIA benchmark under Kimi-K2.6.}
\label{latency}
\end{table}

\subsubsection{Hyperparameter Sensitivity}
Figure~\ref{hyper} investigates the impact of four key hyperparameters, including the activation-cost weight $\lambda$, switching-cost weight $\mu$, teacher-student mixture coefficient $\eta$, and distillation-RL balancing coefficient $\alpha$. Overall, \textsc{CoCA} remains robust across different parameter ranges. The results show that moderate trade-offs between capability cost and allocation stability lead to better performance, while overly aggressive cost regularization or switching constraints may limit effective capability selection. Meanwhile, appropriate settings of $\eta$ and $\alpha$ effectively balance teacher-guided distillation and trajectory-level optimization, highlighting the effectiveness of the dual-level policy learning strategy in \textsc{CoCA}.

\subsubsection{Efficiency Analysis}
We report the execution efficiency comparison on GAIA in Table~\ref{latency}. \textsc{CoCA} consistently reduces inference latency, LLM invocation frequency, and token consumption, indicating that effective capability orchestration can improve both performance and efficiency. The reduction in average activated capabilities per stage further verifies that \textsc{CoCA} performs adaptive capability selection rather than relying on excessive agent or tool execution. The additional training cost comes from the iterative on-policy teacher-guided and trajectory-level policy optimization, which trades training overhead for efficient and reliable inference.

\section{Conclusion}
In this paper, we investigated the combinatorial capability allocation problem for long-horizon multimodal agents. We proposed \textsc{CoCA}, a novel framework that learns capability-subset policies from sparse conditional comparisons. By modeling context-dependent capability utilities, performing autoregressive subset generation, and introducing dual-level on-policy distillation, \textsc{CoCA} addresses the challenges of combinatorial action spaces and distribution mismatch in long-horizon allocation. Trajectory-level reinforcement learning jointly optimizes task performance, activation cost, and allocation stability. Experiments demonstrate that \textsc{CoCA} consistently outperforms existing state-of-the-art methods.




\bibliography{reference}

@article{xie2024osworld,
  title={Osworld: Benchmarking multimodal agents for open-ended tasks in real computer environments},
  author={Xie, Tianbao and Zhang, Danyang and Chen, Jixuan and Li, Xiaochuan and Zhao, Siheng and Cao, Ruisheng and Hua, Toh J and Cheng, Zhoujun and Shin, Dongchan and Lei, Fangyu and others},
  journal={Advances in Neural Information Processing Systems},
  volume={37},
  pages={52040--52094},
  year={2024}
}

@inproceedings{koh2024visualwebarena,
  title={Visualwebarena: Evaluating multimodal agents on realistic visual web tasks},
  author={Koh, Jing Yu and Lo, Robert and Jang, Lawrence and Duvvur, Vikram and Lim, Ming and Huang, Po-Yu and Neubig, Graham and Zhou, Shuyan and Salakhutdinov, Russ and Fried, Daniel},
  booktitle={Proceedings of the 62nd Annual Meeting of the Association for Computational Linguistics (Volume 1: Long Papers)},
  pages={881--905},
  year={2024}
}

@inproceedings{mialon2024gaia,
  title={Gaia: a benchmark for general ai assistants},
  author={Mialon, Gr{\'e}goire and Fourrier, Cl{\'e}mentine and Wolf, Thomas and LeCun, Yann and Scialom, Thomas},
  booktitle={International Conference on Learning Representations},
  volume={2024},
  pages={9025--9049},
  year={2024}
}

@article{yue2024mmmu,
  title={Mmmu-pro: A more robust multi-discipline multimodal understanding benchmark},
  author={Yue, Xiang and Zheng, Tianyu and Ni, Yuansheng and Wang, Yubo and Zhang, Kai and Tong, Shengbang and Sun, Yuxuan and Yu, Botao and Zhang, Ge and Sun, Huan and others},
  journal={arXiv preprint arXiv:2409.02813},
  year={2024}
}

@inproceedings{yao2023react,
  title={ReAct: Synergizing Reasoning and Acting in Language Models},
  author={Yao, Shunyu and Zhao, Jeffrey and Yu, Dian and Du, Nan and Shafran, Izhak and Narasimhan, Karthik and Cao, Yuan},
  booktitle={International Conference on Learning Representations (ICLR)},
  year={2023}
}

@article{schick2023toolformer,
  title={Toolformer: Language models can teach themselves to use tools},
  author={Schick, Timo and Dwivedi-Yu, Jane and Dess{\`\i}, Roberto and Raileanu, Roberta and Lomeli, Maria and Hambro, Eric and Zettlemoyer, Luke and Cancedda, Nicola and Scialom, Thomas},
  journal={Advances in neural information processing systems},
  volume={36},
  pages={68539--68551},
  year={2023}
}

@article{dang2025multi,
  title={Multi-agent collaboration via evolving orchestration},
  author={Dang, Yufan and Qian, Chen and Luo, Xueheng and Fan, Jingru and Xie, Zihao and Shi, Ruijie and Chen, Weize and Yang, Cheng and Che, Xiaoyin and Tian, Ye and others},
  journal={Advances in neural information processing systems},
  volume={38},
  pages={165025--165059},
  year={2025}
}

@inproceedings{jia2026autotool,
  title={AutoTool: Efficient tool selection for large language model agents},
  author={Jia, Jingyi and Li, Qinbin},
  booktitle={Proceedings of the AAAI Conference on Artificial Intelligence},
  volume={40},
  pages={31265--31273},
  year={2026}
}

@inproceedings{jiang2026naviagent,
  title={NaviAgent: Graph\nobreakdash-Driven Bilevel Planning for Scalable Tool Orchestration},
  author={Yan Jiang and Hao Zhou and Lizhong Gu and Tianlong Li and Ruinan Jin and Wanqi Zhou and Ai Han},
  booktitle={Forty-third International Conference on Machine Learning},
  year={2026},
  url={https://openreview.net/forum?id=omRyEle4jZ}
}

@misc{qwen36blog,
    title = {{Qwen3.6-27B}: Flagship-Level Coding in a 27B Dense Model},
    url = {https://qwen.ai/blog?id=qwen3.6-27b},
    author = {{Qwen Team}},
    month = {April},
    year = {2026}
}

@misc{kimik26blog,
    title = {Kimi K2.6: Advancing Open-Source Coding},
    url = {https://www.kimi.com/blog/kimi-k2-6},
    author = {Kimi},
    month = {April},
    year = {2026}
}

@misc{gpt55blog,
    title = {Introducing GPT-5.5},
    url = {https://openai.com/zh-Hans-CN/index/introducing-gpt-5-5/},
    author = {OpenAI},
    month = {April},
    year = {2026}
}

@inproceedings{tian2025mmina,
  title={Mmina: Benchmarking multihop multimodal internet agents},
  author={Tian, Shulin and Zhang, Ziniu and Chen, Liang-Yu and Liu, Ziwei},
  booktitle={Findings of the Association for Computational Linguistics: ACL 2025},
  pages={13682--13697},
  year={2025}
}

@inproceedings{wu2024autogen,
  title={Autogen: Enabling next-gen LLM applications via multi-agent conversations},
  author={Wu, Qingyun and Bansal, Gagan and Zhang, Jieyu and Wu, Yiran and Li, Beibin and Zhu, Erkang and Jiang, Li and Zhang, Xiaoyun and Zhang, Shaokun and Liu, Jiale and others},
  booktitle={First conference on language modeling},
  year={2024}
}

@article{xu2025theagentcompany,
  title={Theagentcompany: benchmarking llm agents on consequential real world tasks},
  author={Xu, Frank Fangzheng and Song, Yufan and Li, Boxuan and Tang, Yuxuan and Jain, Kritanjali and Bao, Mengxue and Wang, Zora and Zhou, Xuhui and Guo, Zhitong and Cao, Murong and others},
  journal={Advances in Neural Information Processing Systems},
  volume={38},
  year={2025}
}

@inproceedings{ong2025routellm,
  title={Route{LLM}: Learning to Route {LLM}s from Preference Data},
  author={Isaac Ong and Amjad Almahairi and Vincent Wu and Wei-Lin Chiang and Tianhao Wu and Joseph E. Gonzalez and M Waleed Kadous and Ion Stoica},
  booktitle={The Thirteenth International Conference on Learning Representations},
  year={2025},
  url={https://openreview.net/forum?id=8sSqNntaMr}
}

@inproceedings{fan2025workflowllm,
  title={Workflowllm: Enhancing workflow orchestration capability of large language models},
  author={Fan, Shengda and Cong, Xin and Fu, Yuepeng and Zhang, Zhong and Zhang, Shuyan and Liu, Yuanwei and Wu, Yesai and Lin, Yankai and Liu, Zhiyuan and Sun, Maosong},
  booktitle={International Conference on Learning Representations},
  volume={2025},
  pages={24498--24525},
  year={2025}
}

@inproceedings{xu2025crab,
    title = "{CRAB}: Cross-environment Agent Benchmark for Multimodal Language Model Agents",
    author = "Xu, Tianqi  and
      Chen, Linyao  and
      Wu, Dai-Jie  and
      Chen, Yanjun  and
      Zhang, Zecheng  and
      Yao, Xiang  and
      Xie, Zhiqiang  and
      Chen, Yongchao  and
      Liu, Shilong  and
      Qian, Bochen  and
      Yang, Anjie  and
      Jin, Zhaoxuan  and
      Deng, Jianbo  and
      Torr, Philip  and
      Ghanem, Bernard  and
      Li, Guohao",
    editor = "Che, Wanxiang  and
      Nabende, Joyce  and
      Shutova, Ekaterina  and
      Pilehvar, Mohammad Taher",
    booktitle = "Findings of the Association for Computational Linguistics: ACL 2025",
    month = jul,
    year = "2025",
    address = "Vienna, Austria",
    publisher = "Association for Computational Linguistics",
    url = "https://aclanthology.org/2025.findings-acl.1113/",
    doi = "10.18653/v1/2025.findings-acl.1113",
    pages = "21607--21647",
    ISBN = "979-8-89176-256-5"
}

@inproceedings{liu2025toolace,
  title={Toolace: Winning the points of llm function calling},
  author={Liu, Weiwen and Huang, Xu and Zeng, Xingshan and Yu, Shuai and Li, Dexun and Wang, Shuai and Gan, Weinan and Liu, Zhengying and Yu, Yuanqing and WANG, Zezhong and others},
  booktitle={International conference on learning representations},
  volume={2025},
  pages={41359--41381},
  year={2025}
}

@article{zhang2025agentorchestra,
  title={Agentorchestra: A hierarchical multi-agent framework for general-purpose task solving},
  author={Zhang, Wentao and Cui, Ce and Zhao, Yilei and Hu, Rui and Liu, Yang and Zhou, Yahui and An, Bo},
  journal={arXiv e-prints},
  pages={arXiv--2506},
  year={2025}
}

@inproceedings{zhang2026mtrouter,
  title={MTRouter: Cost-Aware Multi-Turn LLM Routing with History--Model Joint Embeddings},
  author={Zhang, Yiqun and Li, Hao and Wang, Zihan and Feng, Shi and Yang, Xiaocui and Wang, Daling and Zhang, Bo and Bai, Lei and Hu, Shuyue},
  booktitle={Proceedings of the 64th Annual Meeting of the Association for Computational Linguistics (Volume 1: Long Papers)},
  pages={44206--44226},
  year={2026}
}

@inproceedings{xi2025agentgym,
  title={Agentgym: Evaluating and training large language model-based agents across diverse environments},
  author={Xi, Zhiheng and Ding, Yiwen and Chen, Wenxiang and Hong, Boyang and Guo, Honglin and Wang, Junzhe and Guo, Xin and Yang, Dingwen and Liao, Chenyang and He, Wei and others},
  booktitle={Proceedings of the 63rd Annual Meeting of the Association for Computational Linguistics (Volume 1: Long Papers)},
  pages={27914--27961},
  year={2025}
}

@article{xu2026evolution,
  title={The evolution of tool use in llm agents: From single-tool call to multi-tool orchestration},
  author={Xu, Haoyuan and Li, Chang and Ma, Xinyan and Ou, Xianhao and Zhang, Zihan and He, Tao and Liu, Xiangyu and Wang, Zixiang and Liang, Jiafeng and Chu, Zheng and others},
  journal={arXiv preprint arXiv:2603.22862},
  year={2026}
}

@article{zhang2026evoroute,
  title={EvoRoute: Experience-Driven Self-Routing LLM Agent Systems},
  author={Zhang, Guibin and Yu, Haiyang and Yang, Kaiming and Wu, Bingli and Huang, Fei and Li, Yongbin and Yan, Shuicheng},
  journal={arXiv preprint arXiv:2601.02695},
  year={2026}
}

@inproceedings{niu2026routing,
  title={Routing with Generated Data: Annotation-Free LLM Skill Estimation and Expert Selection},
  author={Niu, Tianyi and Chen, Justin and Winata, Genta Indra and Zhang, Shi-Xiong and Chakraborty, Supriyo and Sahu, Sambit and Zhang, Yue and Stengel-Eskin, Elias and Bansal, Mohit},
  booktitle={Proceedings of the 64th Annual Meeting of the Association for Computational Linguistics (Volume 1: Long Papers)},
  pages={32441--32466},
  year={2026}
}

@inproceedings{zhang2026agentrouter,
  title={Agentrouter: A knowledge-graph-guided llm router for collaborative multi-agent question answering},
  author={Zhang, Zheyuan and Shi, Kaiwen and Yuan, Zhengqing and Wang, Zehong and Ma, Tianyi and Murugesan, Keerthiram and Galassi, Vincent and Zhang, Chuxu and Ye, Yanfang},
  booktitle={Proceedings of the 64th Annual Meeting of the Association for Computational Linguistics (Volume 1: Long Papers)},
  pages={788--809},
  year={2026}
}

@article{zheng2026skillselect,
  title={SkillSelect-Serve: Budget-Controllable and QoS-Aware Skill Service Recommendation and Composition for Small LLM Agents},
  author={Zheng, Jingyuan and Wang, Dongjing and Zhang, Xin and Huang, Butian and Zhang, Haiping and Yu, Dongjin and Deng, Shuguang},
  journal={arXiv preprint arXiv:2607.00011},
  year={2026}
}

@article{li2026organizing,
  title={Organizing, orchestrating, and benchmarking agent skills at ecosystem scale},
  author={Li, Hao and Mu, Chunjiang and Chen, Jianhao and Ren, Siyue and Cui, Zhiyao and Zhang, Yiqun and Bai, Lei and Hu, Shuyue},
  journal={arXiv preprint arXiv:2603.02176},
  year={2026}
}

@inproceedings{ruan2026aorchestra,
  title={Aorchestra: Automating sub-agent creation for agentic orchestration},
  author={Ruan, Jianhao and Xu, Zhihao and Peng, Yiran and Ren, Fashen and Yu, Zhaoyang and Liang, Xinbing and Xiang, Jinyu and Chen, Yongru and Liu, Bang and Wu, Chenglin and others},
  booktitle={Forty-third International Conference on Machine Learning},
  year={2026}
}

@inproceedings{yao2026ace,
  title={ACE-Router: Generalizing History-Aware Routing from MCP Tools to the Agent Web},
  author={Yao, Zhiyuan and Xu, Zishan and Guo, Yifu and Han, Zhiguang and Yang, Cheng and Zhang, Shuo and Zhang, Weinan and Zeng, Xingshan and Liu, Weiwen},
  booktitle={Proceedings of the 64th Annual Meeting of the Association for Computational Linguistics (Volume 1: Long Papers)},
  pages={6224--6240},
  year={2026}
}

@inproceedings{wu2025agentic,
  title={Agentic reasoning: A streamlined framework for enhancing llm reasoning with agentic tools},
  author={Wu, Junde and Zhu, Jiayuan and Liu, Yuyuan and Xu, Min and Jin, Yueming},
  booktitle={Proceedings of the 63rd Annual Meeting of the Association for Computational Linguistics (Volume 1: Long Papers)},
  pages={28489--28503},
  year={2025}
}

@inproceedings{zhao2025cola,
  title={COLA: Collaborative Multi-Agent Framework with Dynamic Task Scheduling for GUI Automation},
  author={Zhao, Di and Ma, Longhui and Wang, Siwei and Wang, Miao and Lv, Zhao},
  booktitle={Proceedings of the 2025 Conference on Empirical Methods in Natural Language Processing},
  pages={4570--4593},
  year={2025}
}

@inproceedings{wang2026always,
    title = "Do We Always Need Query-Level Workflows? Rethinking Agentic Workflow Generation for Multi-Agent Systems",
    author = "Wang, Zixu  and
      Xu, Bingbing  and
      Yuan, Yige  and
      Shen, Huawei  and
      Cheng, Xueqi",
    editor = "Liakata, Maria  and
      Moreira, Viviane P.  and
      Zhang, Jiajun  and
      Jurgens, David",
    booktitle = "Findings of the {A}ssociation for {C}omputational {L}inguistics: {ACL} 2026",
    month = jul,
    year = "2026",
    address = "San Diego, California, United States",
    publisher = "Association for Computational Linguistics",
    url = "https://aclanthology.org/2026.findings-acl.254/",
    doi = "10.18653/v1/2026.findings-acl.254",
    pages = "5149--5165",
    ISBN = "979-8-89176-395-1"
}

@inproceedings{wang2026fusionflow,
  title={FusionFlow: Enabling Deep Structural Exploration for Automated Agentic Workflow Generation},
  author={Wang, Xiang and Yang, Zongtao and Hong, Zhuojian and Zhang, Shuhao and Wei, Wei},
  booktitle={Proceedings of the 64th Annual Meeting of the Association for Computational Linguistics (Volume 1: Long Papers)},
  pages={27718--27760},
  year={2026}
}

@inproceedings{dekoninck2025unified,
  title={A Unified Approach to Routing and Cascading for LLMs},
  author={Dekoninck, Jasper and Baader, Maximilian and Vechev, Martin},
  booktitle={International Conference on Machine Learning},
  pages={12987--13010},
  year={2025},
  organization={PMLR}
}

@inproceedings{faghih2025tool,
  title={Tool Preferences in Agentic LLMs are Unreliable},
  author={Faghih, Kazem and Wang, Wenxiao and Cheng, Yize and Bharti, Siddhant and Sriramanan, Gaurang and Balasubramanian, Sriram and Hosseini, Parsa and Feizi, Soheil},
  booktitle={Proceedings of the 2025 Conference on Empirical Methods in Natural Language Processing},
  pages={20965--20980},
  year={2025}
}

@inproceedings{li2025adaptive,
  title={Adaptive tool use in large language models with meta-cognition trigger},
  author={Li, Wenjun and Li, Dexun and Dong, Kuicai and Zhang, Cong and Zhang, Hao and Liu, Weiwen and Wang, Yasheng and Tang, Ruiming and Liu, Yong},
  booktitle={Proceedings of the 63rd Annual Meeting of the Association for Computational Linguistics (Volume 1: Long Papers)},
  pages={13346--13370},
  year={2025}
}

@inproceedings{zhang2025toolexpnet,
  title={Toolexpnet: Optimizing multi-tool selection in llms with similarity and dependency-aware experience networks},
  author={Zhang, Zijing and Chen, Zhanpeng and Zhu, He and Chen, Ziyang and Du, Nan and Li, Xiaolong},
  booktitle={Findings of the Association for Computational Linguistics: ACL 2025},
  pages={15706--15722},
  year={2025}
}

@inproceedings{ding2025best,
  title={BEST-Route: Adaptive LLM Routing with Test-Time Optimal Compute},
  author={Ding, Dujian and Mallick, Ankur and Zhang, Shaokun and Wang, Chi and Madrigal, Daniel and Garcia, Mirian Del Carmen Hipolito and Xia, Menglin and Lakshmanan, Laks VS and Wu, Qingyun and R{\"u}hle, Victor},
  booktitle={International Conference on Machine Learning},
  pages={13870--13884},
  year={2025},
  organization={PMLR}
}

@inproceedings{xiang2026llm,
  title={LLM-as-Scheduler: Agentic Workflow Dynamic Scheduling},
  author={Xiang, Dawei and Chu, Kexin and Xu, Wenyan and Zhang, Wenhui and Zhang, Wei},
  booktitle={Proceedings of the 64th Annual Meeting of the Association for Computational Linguistics (Volume 1: Long Papers)},
  pages={12752--12763},
  year={2026}
}

@inproceedings{liu2026evolving,
  title={Evolving Agentic Workflow Driven by Human-Agent Collaboration},
  author={Liu, Yuxin and Zhang, Jinxuan and Peng, Yuezhang and Zhou, Hefeng and Wang, Xiangfeng and Lou, Jiong and Wu, Chentao and Li, Jie and Qu, Jingjing and Lu, Chaochao},
  booktitle={Findings of the Association for Computational Linguistics: ACL 2026},
  pages={24960--24969},
  year={2026}
}

@article{nayan2026multi,
  title={Multi-Agent Routing as Set-Valued Prediction: A WildChat Benchmark and Cost-Aware Evaluation},
  author={Nayan Bala, Ananto and Shah, Faisal Muhammad},
  journal={arXiv e-prints},
  pages={arXiv--2606},
  year={2026}
}

@article{zhong2026sod,
  title={Sod: Step-wise on-policy distillation for small language model agents},
  author={Zhong, Qiyong and Zheng, Mao and Song, Mingyang and Lin, Xin and Sun, Jie and Jiang, Houcheng and Wang, Xiang and Fang, Junfeng},
  journal={arXiv preprint arXiv:2605.07725},
  year={2026}
}

@inproceedings{zhang2026fast,
  title={Fast and effective on-policy distillation from reasoning prefixes},
  author={Zhang, Dongxu and Yang, Zhichao and Janghorbani, Sepehr and Han, Jun and Ressler II, Andrew and Qian, Qian and Lyng, Gregory D and Batra, Sanjit Singh and Tillman, Robert E},
  booktitle={Findings of the Association for Computational Linguistics: ACL 2026},
  pages={25553--25569},
  year={2026}
}

@inproceedings{zhang2025beyond,
  title={Beyond Bradley-Terry Models: A General Preference Model for Language Model Alignment},
  author={Zhang, Yifan and Zhang, Ge and Wu, Yue and Xu, Kangping and Gu, Quanquan},
  booktitle={International Conference on Machine Learning},
  pages={76939--76965},
  year={2025},
  organization={PMLR}
}

@inproceedings{cho2025policy,
  title={Policy-labeled Preference Learning: Is Preference Enough for RLHF?},
  author={Cho, Taehyun and Ju, Seokhun and Han, Seungyub and Kim, Dohyeong and Lee, Kyungjae and Lee, Jungwoo},
  booktitle={International Conference on Machine Learning},
  pages={10524--10553},
  year={2025},
  organization={PMLR}
}

@article{yang2026opid,
  title={Opid: On-policy skill distillation for agentic reinforcement learning},
  author={Yang, Shuo and Wu, Jinyang and Lu, Zhengxi and Shen, Yuhao and Zhang, Fan and Feng, Lang and Zhang, Shuai and Luo, Haoran and Lian, Zheng and Wen, Zhengqi and others},
  journal={arXiv preprint arXiv:2606.26790},
  year={2026}
}

@article{wu2026seed,
  title={SEED: Self-Evolving On-Policy Distillation for Agentic Reinforcement Learning},
  author={Wu, Jinyang and Yang, Shuo and Lu, Zhengxi and Zhang, Fan and Shen, Yuhao and Feng, Lang and Luo, Haoran and Lian, Zheng and Zhang, Shuai and Wen, Zhengqi and others},
  journal={arXiv preprint arXiv:2607.14777},
  year={2026}
}

\newpage

\appendix

\section{Proof of Theorem~\ref{thm_preference_subset_stability}}
\begin{proof}
We first bound the conditional utility estimation error using the local strong convexity of the anchored Bradley--Terry risk, and then propagate this error through the autoregressive factorization to bound the discrepancy between the oracle and learned subset policies. For a candidate vector $u$, the Hessian of the conditional Bradley--Terry risk is
\begin{align}
\nabla^2\mathcal{R}_c(u) \!=\!\!\! \sum_{(a,b)\in\mathcal{E}_c}\!\! w_{ab}(c)\sigma(u_a-u_b)(1 \!-\! \sigma(u_a \!-\! u_b))b_{ab}b_{ab}^{\top}.
\label{eq:theory-bt-hessian}
\end{align}

Along the line segment between $u_c^\star$ and $\widehat u_c$, the bounded-logit assumption implies
\begin{align}
\sigma(u_a-u_b)
\bigl(1-\sigma(u_a-u_b)\bigr)
\geq
\omega_B.
\end{align}

Consequently,
\begin{align}
\nabla^2\mathcal{R}_c(u)
\succeq
\omega_B L_c
\succeq
\omega_B\gamma_c I
\label{eq:theory-risk-strong-convexity}
\end{align}
throughout this segment. The population model is correctly specified, so $\nabla\mathcal{R}_c(u_c^\star)=0$. Let $\Delta_c=\widehat u_c-u_c^\star$. By the fundamental theorem of calculus,
\begin{align}
\! \nabla\mathcal{R}_c(\widehat u_c) \!-\! \nabla\mathcal{R}_c(u_c^\star) \!=\! \left[\int_0^1 \!\! \nabla^2\mathcal{R}_c (u_c^\star \!+\! t\Delta_c) dt \right]\Delta_c. \!
\label{eq:theory-gradient-integral}
\end{align}

Taking the inner product with $\Delta_c$ and applying Eq.~\eqref{eq:theory-risk-strong-convexity} yields
\begin{align}
\omega_B\gamma_c\|\Delta_c\|_2^2
&\leq
\left\langle
\nabla\mathcal{R}_c(\widehat u_c)
-
\nabla\mathcal{R}_c(u_c^\star),
\Delta_c
\right\rangle \nonumber
\\
&\leq
\left\|
\nabla\mathcal{R}_c(\widehat u_c)
\right\|_2
\|\Delta_c\|_2=
\rho_c\|\Delta_c\|_2,
\label{eq:theory-utility-error-intermediate}
\end{align}
where the second inequality is Cauchy--Schwarz. If $\Delta_c=0$, $\| \widehat u_c-u_c^\star\|_2 \leq \frac{\rho_c}{\omega_B\gamma_c}$ holds trivially; otherwise, dividing by $\omega_B\gamma_c\|\Delta_c\|_2$ proves it. This establishes the first claim of the theorem. We next translate the conditional utility error into a discrepancy between the corresponding autoregressive include-or-skip decisions. At canonical position $i$ and prefix $S_{t,<i}$, let $p_i^\star=\sigma(u_{(s,S_{t,<i})}^\star(k_i))$, $\widehat p_i=\sigma(\widehat u_{(s,S_{t,<i})}(k_i))$. Since $\sup_x\sigma'(x)=1/4$, we have
\begin{align} \label{eq:theory-conditional-policy-perturbation}
\left|p_i^\star-\widehat p_i\right|
&\leq
\frac{1}{4}
\left|
u_{(s,S_{t,<i})}^\star(k_i)
-
\widehat u_{(s,S_{t,<i})}(k_i)
\right| \nonumber
\\
&\leq
\frac{1}{4}
\frac{
\rho_{(s,S_{t,<i})}
}{
\omega_B\gamma_{(s,S_{t,<i})}
}.
\end{align}

For Bernoulli decisions, $\left|p_i^\star-\widehat p_i\right|$ equals the conditional total-variation distance. It remains to aggregate these conditional discrepancies across the $K$ autoregressive decisions. To this end, we use a standard hybrid argument. Define hybrid subset distributions $M_0,\ldots,M_K$, where $M_i$ uses the learned teacher conditionals for the first $i$ canonical decisions and the oracle conditionals thereafter. Thus, $M_0=\pi^\star$ and $M_K=\pi_T$. The triangle inequality gives
\begin{align}
D_{\mathrm{TV}}(\pi^\star,\pi_T)
\leq
\sum_{i=1}^{K}
D_{\mathrm{TV}}(M_{i-1},M_i).
\label{eq:theory-preference-hybrid-triangle}
\end{align}

Adjacent hybrids share the learned teacher prefix distribution $d_i^{\pi_T}(\cdot\mid s)$ and the same oracle continuation after position $i$. Contractivity of total variation under this common continuation kernel therefore gives
\begin{align}
\begin{split}
D_{\mathrm{TV}}(M_{i-1},M_i)
\leq
\mathbb{E}_{S_{t,<i}\sim d_i^{\pi_T}(\cdot\mid s)}
\left[
\left|p_i^\star-\widehat p_i\right|
\right].
\end{split}
\label{eq:theory-preference-adjacent-hybrid}
\end{align}

By substituting Eq.~\eqref{eq:theory-conditional-policy-perturbation} into Eq.~\eqref{eq:theory-preference-adjacent-hybrid} and summing over $i$ proves Eq.~\eqref{eq:theory-preference-subset-bound}.
\end{proof}

\section{Proof of Theorem~\ref{thm_student_prefix_composition}}
\begin{proof}
Fix $s$ and let $z_{1:K}$ be the canonical binary construction sequence corresponding to a capability subset. For each $i\in\{0,\ldots,K\}$, define a hybrid sequence distribution
\begin{align}
H_i(z_{1:K}\mid s)
=
\left[
\prod_{j=1}^{i}
Q_j^s(z_j\mid S_{<j})
\right]
\left[
\prod_{j=i+1}^{K}
P_j^s(z_j\mid S_{<j})
\right].
\label{eq:theory-hybrid-distribution}
\end{align}

The empty product convention gives $H_0=P^s$ and $H_K=Q^s$. By the triangle inequality for total variation,
\begin{align}
D_{\mathrm{TV}}(P^s,Q^s)
\leq
\sum_{i=1}^{K}
D_{\mathrm{TV}}(H_{i-1},H_i).
\label{eq:theory-hybrid-triangle}
\end{align}

The two adjacent hybrids use the same student conditionals for positions $1,\ldots,i-1$, so their common prefix is distributed as $d_i^{\pi_\theta}(\cdot\mid s)$. Conditioned on a fixed prefix $S_{t,<i}$, they differ only at position $i$: $H_{i-1}$ uses $P_i^s$, whereas $H_i$ uses $Q_i^s$. Both hybrids use the same teacher continuation kernel after position $i$. Applying the same continuation kernel cannot increase total variation, and therefore $D_{\mathrm{TV}}(H_{i-1},H_i)
\leq
\mathbb{E}_{S_{t,<i}\sim d_i^{\pi_\theta}(\cdot\mid s)}
\left[
D_{\mathrm{TV}}\left(
P_i^s(\cdot\mid S_{t,<i}),
Q_i^s(\cdot\mid S_{t,<i})
\right)
\right]$. By substituting it into Eq.~\eqref{eq:theory-hybrid-triangle} proves the first inequality in Eq.~\eqref{eq:theory-local-global-bound}. For the second inequality, Pinsker's inequality gives, for every prefix,
\begin{align}
D_{\mathrm{TV}}\!\left(P_i^s,Q_i^s\right)
\leq
\sqrt{
\frac{1}{2}
D_{\mathrm{KL}}\!\left(P_i^s\Vert Q_i^s\right)
}.
\label{eq:theory-conditional-pinsker}
\end{align}

Taking expectations over student prefixes and applying Jensen's inequality yields
\begin{align}
\mathbb{E}_{d_i^{\pi_\theta}}
\left[
D_{\mathrm{TV}}\!\left(P_i^s,Q_i^s\right)
\right]
\leq
\sqrt{\frac{\delta_i(s)}{2}}.
\label{eq:theory-expected-conditional-tv}
\end{align}

Summing Eq.~\eqref{eq:theory-expected-conditional-tv} over $i$ proves the second inequality. Finally, Cauchy--Schwarz gives
\begin{align}
\sum_{i=1}^{K}\sqrt{\frac{\delta_i(s)}{2}}
\leq
\sqrt{
K\sum_{i=1}^{K}\frac{\delta_i(s)}{2}
},
\end{align}
which proves the final inequality.
\end{proof}

\section{Proof of Theorem~\ref{thm_nested_performance}}
\begin{proof}
For each position $i$, define the deployment conditional KL as
\begin{align}
\bar\delta_i
=
\mathbb{E}_{s\sim\bar d^{\pi_\theta}}
\mathbb{E}_{S_{t,<i}\sim d_i^{\pi_\theta}(\cdot\mid s)}
\left[
D_{\mathrm{KL}}\!\left(
P_i^s(\cdot\mid S_{t,<i})
\Vert
Q_i^s(\cdot\mid S_{t,<i})
\right)
\right].
\label{eq:theory-deployment-conditional-kl}
\end{align}

By introduce the density ratios $r_{\mathrm{env}}(s)
=
\frac{\bar d^{\pi_\theta}(s)}{\nu(s)}$ and $r_{\mathrm{set},i}(S_{t,<i}\mid s)
=
\frac{
d_i^{\pi_\theta}(S_{t,<i}\mid s)
}{
q_i(S_{t,<i}\mid s)
}$, the changing measure from the deployment distributions to the training distributions gives
\begin{align}
\bar\delta_i
=&\ 
\mathbb{E}_{s\sim\nu}
\mathbb{E}_{S_{t,<i}\sim q_i(\cdot\mid s)}
\Big[
r_{\mathrm{env}}(s) 
r_{\mathrm{set},i}(S_{t,<i}\mid s) \nonumber
\\
&\cdot
D_{\mathrm{KL}}\!\left(
P_i^s(\cdot\mid S_{t,<i})
\Vert
Q_i^s(\cdot\mid S_{t,<i})
\right)
\Big] \leq
C_{\mathrm{e}}C_{\mathrm{s},i}\epsilon_i.
\label{eq:theory-change-of-measure}
\end{align}

Applying the first two inequalities of Theorem~\ref{thm_student_prefix_composition} at each state and averaging over $\bar d^{\pi_\theta}$ gives
\begin{align}
\mathbb{E}_{s\sim\bar d^{\pi_\theta}}
[
& D_{\mathrm{TV}}\!\left(
\pi_T(\cdot\mid s),
\pi_\theta(\cdot\mid s)
\right)
]\leq
\sum_{i=1}^{K}
\mathbb{E}_{s\sim\bar d^{\pi_\theta}}
\left[
\sqrt{\frac{\delta_i(s)}{2}}
\right] \nonumber
\\
&\quad\leq
\sum_{i=1}^{K}
\sqrt{\frac{\bar\delta_i}{2}}
\leq
\sum_{i=1}^{K}
\sqrt{
\frac{
C_{\mathrm{e}}C_{\mathrm{s},i}\epsilon_i
}{2}
},
\label{eq:theory-average-subset-tv}
\end{align}
where the second inequality follows from Jensen's inequality. Since total variation is at most one, the left-hand side is also bounded by the minimum appearing in Eq.~\eqref{eq:theory-nested-performance}. It remains to translate subset discrepancy into return discrepancy. The finite-horizon performance-difference identity, expanded along the student state distribution, is
\begin{align}
\begin{split}
J(\pi_\theta) \!-\! J(\pi_T) \!=\! \sum_{t=0}^{T-1}
\mathbb{E}_{s\sim d_t^{\pi_\theta}}
\mathbb{E}_{G\sim\pi_\theta(\cdot\mid s)} [A_t^{\pi_T}(s,G)].
\end{split}
\label{eq:theory-performance-difference}
\end{align}

For every state, $\mathbb{E}_{G\sim\pi_T(\cdot\mid s)}
\left[
A_t^{\pi_T}(s,G)
\right]
=0$. Subtracting this zero term and using the dual characterization of total variation together with $|A_t^{\pi_T}(s,G)| \leq \bar A_T$ yields
\begin{align}
&\ |\mathbb{E}_{G\sim\pi_\theta}[A_t^{\pi_T}(s,G)]| \nonumber \\
=
&\ \left|
\mathbb{E}_{G\sim\pi_\theta}
[A_t^{\pi_T}(s,G)]
-
\mathbb{E}_{G\sim\pi_T}
[A_t^{\pi_T}(s,G)]
\right| \nonumber \\
\leq 
&\ 2\bar A_T
D_{\mathrm{TV}}\!\left(
\pi_T(\cdot\mid s),
\pi_\theta(\cdot\mid s)
\right).
\label{eq:theory-advantage-tv}
\end{align}

Taking absolute values in Eq.~\eqref{eq:theory-performance-difference}, applying Eq.~\eqref{eq:theory-advantage-tv}, and using the definition of $\bar d^{\pi_\theta}$ gives
\begin{align}
\left|
J(\pi_\theta)-J(\pi_T)
\right|
&\leq
2\bar A_T
\sum_{t=0}^{T-1}
\mathbb{E}_{s\sim d_t^{\pi_\theta}}
\left[
D_{\mathrm{TV}}(\pi_T,\pi_\theta)
\right] \nonumber
\\
&=
2T\bar A_T
\mathbb{E}_{s\sim\bar d^{\pi_\theta}}
\left[
D_{\mathrm{TV}}(\pi_T,\pi_\theta)
\right].
\label{eq:theory-return-via-tv}
\end{align}

Substituting Eq.~\eqref{eq:theory-average-subset-tv} and the unit upper bound on total variation proves Eq.~\eqref{eq:theory-nested-performance}.
\end{proof}

\section{Algorithm} \label{app:training_algorithm}

\textsc{CoCA} alternates between conditional utility learning and student-policy optimization. At each outer iteration, the student collects trajectories and visited states, from which student-induced partial subsets are extracted for preference collection. The resulting decision contexts are evaluated under the fixed continuation policy $\pi_{\theta_n}$ to update $v_\phi$, which induces $\pi_T$. The student is then optimized on the visited states using mixed-prefix distillation and trajectory-level policy optimization. The two models are updated separately: $\mathcal{L}_{\mathrm{p}}$ updates $\phi$, while $\mathcal{L}_{\mathrm{PG}}+\alpha\mathcal{L}_{\mathrm{D}}$ updates $\theta$.

Let $K$ be the number of capabilities and $T$ the trajectory length. The autoregressive allocator performs exactly $K$ conditional policy evaluations per stage and $TK$ per trajectory, excluding capability execution and model-dependent forward costs. During training, let $N_s$ be the number of sampled states, $N_p$ the average number of sampled partial subsets per state, and $Q$ the number of comparisons per state--subset condition. Preference collection requires $\mathcal{O}(N_sN_pQ)$ teacher queries, avoiding enumeration of the $2^K$ subset space. Distillation evaluates one local KL term at each autoregressive position and therefore requires $\mathcal{O}(K)$ decision-level evaluations per sampled state. At inference, only $\pi_\theta$ is retained; no teacher queries, utility-model evaluations, preference collection, or online updates are required.


\begin{algorithm}[t]
\caption{Training Procedure of \textsc{CoCA}}
\label{alg:coca_training}
\begin{algorithmic}[1]
\REQUIRE Capability library $\mathcal{K}$, student policy $\pi_\theta$, utility model $v_\phi$, preference teacher $\mathcal{O}$
\STATE Initialize $\theta$ and $\phi$
\FOR{each outer iteration $n$}

    \STATE Roll out $\pi_{\theta_n}$ and collect trajectories $\xi$
    
    \FOR{each visited state $s_t$}
        \FOR{each autoregressive position $i=1,\ldots,K$}
            \STATE Record the student-induced prefix $S_{t,<i}$
        \ENDFOR
    \ENDFOR
    
    \STATE Query $\mathcal{O}$ on collected $(s_t,S_{t,<i})$ contexts under fixed $\pi_{\theta_n}$
    \STATE Update $v_\phi$ by minimizing $\mathcal{L}_{p}$

    \STATE Construct teacher policy $\pi_T$ from $v_\phi$
    \STATE Set mixture coefficient $\eta$

    \FOR{each collected state $s_t$}
        \FOR{each capability position $i=1,\ldots,K$}
            \STATE Generate prefix using
            $\eta\pi_T+(1-\eta)\pi_{\theta_n}$
        \ENDFOR
    \ENDFOR

    \STATE Compute $\mathcal{L}_{\mathrm{D}}$ and $\mathcal{L}_{\mathrm{PG}}$
    \STATE Update $\theta$ by minimizing
    $\mathcal{L}_{\mathrm{PG}}+\alpha\mathcal{L}_{\mathrm{D}}$

\ENDFOR
\RETURN $\pi_\theta$
\end{algorithmic}
\end{algorithm}

\section{Canonical Capability Ordering} \label{app:canonical_set_construction}

The capabilities in $\mathcal{K}$ may exhibit dependencies induced by their information flow. We represent these relations through a predefined dependency-aware protocol and obtain a fixed canonical order $k_1 \prec k_2 \prec \cdots \prec k_K$ by linearizing the corresponding partial order. Capabilities without explicit dependencies are ordered using a deterministic rule, and the same order is used by the teacher and student throughout training and inference.

The canonical order defines the autoregressive factorization of the subset action rather than an intrinsic ranking of capability importance. The policy is not restricted to selecting a prefix: every subset $G\subseteq\mathcal{K}$ is uniquely represented by $(z_1,\ldots,z_K)$, where $z_{t,i}=\mathbb{I}[k_i\in G]$, and can be constructed through the include-or-skip decisions in Section~\ref{sec:autoregressive-policy}. Conditioning on the selected prefix captures complementarity and redundancy while preserving the complete subset space. During execution, unselected capabilities are omitted and independent capabilities may be scheduled in parallel, while the dependency relations remain respected.

\section{Preference Collection} \label{app:conditional_preference_collection}

The comparison types, labels, and Bradley--Terry objective follow Section~\ref{sec:conditional-utility}. Each comparison is conditioned on a state $s$ and partial subset $S$, since a capability's marginal value depends on both the current demand and the capabilities already selected. Capability-versus-null comparisons supervise include-or-skip decisions, whereas capability-versus-capability comparisons distinguish the relative marginal values of non-null candidates. The comparison contexts are drawn from student-generated trajectories and autoregressive prefixes and are refreshed as the student policy evolves. Within each outer iteration, the continuation policy is fixed so that candidates are evaluated under the same downstream behavior. The observed comparisons form a condition-specific graph anchored by the null alternative. The anchor removes additive utility ambiguity, while sufficient connectivity improves utility identifiability and teacher-policy stability. As characterized by Theorem~\ref{thm_preference_subset_stability}, the resulting policy error depends on both preference-fitting accuracy and the conditioning of the comparison graph through $\gamma_c$.

\subsection{Teacher Utility Estimation}
The theoretical marginal advantage $A_{k,\mathrm{task}}^\pi(s,S)$ measures the utility change induced by activating an additional capability under a given state and partial capability set. In practice, we approximate this cost-adjusted marginal utility through teacher-guided counterfactual evaluation.

Specifically, we use GPT-5.5 as the teacher model to evaluate candidate capability choices under identical task contexts. Given a state $s$ and a partial capability set $S$, the teacher assesses the relative utility between activating a candidate capability $k$ and keeping the current subset unchanged. The evaluation considers the same task description, execution history, and available capability outputs, ensuring that the comparison reflects the marginal contribution of the candidate capability. For a candidate capability $k$, we compute the estimated cost-adjusted utility difference:
\begin{align}
\Delta U_{k,\lambda}(s,S) = \hat U_{k,\lambda}(s,S)-\hat U_{\varnothing,\lambda}(s,S),
\end{align}
where $\hat U_{k,\lambda}$ denotes the teacher-estimated utility and the null option provides the zero-utility reference following the null-anchored formulation. Pairwise preference labels are then constructed according to the relative utility:
\begin{align}
y=
\begin{cases}
1, & \Delta U_{i,\lambda}>\Delta U_{j,\lambda},\\
0, & \Delta U_{i,\lambda}<\Delta U_{j,\lambda},\\
0.5, & \Delta U_{i,\lambda}=\Delta U_{j,\lambda}.
\end{cases}
\end{align}

We further assign a confidence weight $w$ according to the utility margin, where larger margins indicate more reliable preferences and near-tie comparisons receive lower weights. These weighted preference pairs are used to train the conditional utility model $v_\phi$, which distills teacher evaluations into a deployable capability utility estimator.

\section{Detailed Dataset Descriptions}

We evaluate \textsc{CoCA} on four benchmarks covering complementary forms of multimodal interaction, tool-augmented reasoning, and long-horizon task execution. 

\textbf{OSWorld} provides realistic computer-use tasks in operating-system environments, where agents interact with graphical user interfaces to complete multi-step objectives. Successful execution requires visual perception, interface grounding, action planning, and persistent interaction with applications, making OSWorld particularly suitable for evaluating stage-wise capability allocation under changing observations and execution states. 

\textbf{VisualWebArena} evaluates multimodal web agents on realistic website navigation and interaction tasks. Agents must interpret visually rendered webpages, locate relevant interface elements, retrieve information, and execute sequential actions across multiple pages, testing the coordination of perception, retrieval, reasoning, and action capabilities.

\textbf{GAIA} evaluates general-purpose AI assistants on real-world questions that often require multi-step reasoning, external information acquisition, and the use of heterogeneous tools. Its tasks vary substantially in the capabilities needed for successful completion, providing a natural setting for assessing whether an allocator can activate specialized capabilities selectively rather than relying on a fixed tool configuration. 

\textbf{MMMU-Pro} focuses on expert-level multimodal reasoning over complex visual and textual inputs. The benchmark includes questions that require fine-grained visual understanding, domain knowledge, and multi-step inference, offering a complementary evaluation of capability allocation in non-interactive but reasoning-intensive scenarios.

Together, these benchmarks span interactive and non-interactive environments, visual grounding and abstract reasoning, and short- to long-horizon decision processes. This diversity allows us to evaluate whether \textsc{CoCA} can adapt capability subsets to heterogeneous task demands while balancing task performance, activation cost, and allocation stability.

\section{Comparison Baselines}

We compare \textsc{CoCA} with representative baselines covering reasoning--action interaction, learned tool invocation, multi-agent orchestration, efficient tool selection, and large-scale toolchain planning. 

Since existing baselines are originally designed for different tool-use or agent orchestration scenarios rather than explicit subset-based capability allocation, we adapt them to our unified evaluation protocol for fair comparison. Specifically, all methods operate on the same capability library with $K=8$ capabilities, identical stage definitions, and the same execution environment. Each baseline is required to produce a capability subset at each stage, which is then executed by the shared executor and evaluated using identical task and allocation metrics. The specific adaptation procedure of each baseline is described below.

\textsc{ReAct}~\citep{yao2023react} integrates reasoning and acting within a unified iterative trajectory. At each step, the language model generates an intermediate reasoning trace to update its task understanding and action plan, and then executes an environment action or invokes an external tool to obtain new observations. These observations are incorporated into subsequent reasoning, enabling the agent to revise its plan and recover from intermediate failures. ReAct is a strong general-purpose baseline for interactive and long-horizon tasks because it supports closed-loop reasoning over dynamically changing environments. However, tool or capability decisions are made sequentially by the language model at individual interaction steps, without explicitly representing a stage-wise subset action or modeling the conditional utility among jointly activated capabilities. For fair comparison, we adapt ReAct by treating the tool invocation decisions generated within each stage as the activated capability subset under the unified execution protocol.

\textsc{Toolformer}~\citep{schick2023toolformer} trains a language model to use external APIs through self-supervised tool-use annotations. Starting from a small number of demonstrations for each API, it samples candidate tool calls in unlabeled text, executes the corresponding APIs, and retains calls that improve the model's language-modeling objective. The resulting model learns when a tool should be invoked, which tool to invoke, how to generate its arguments, and how to incorporate the returned result into subsequent token generation. Toolformer therefore provides a representative learned tool-selection baseline that avoids manually labeling every tool call. Nevertheless, its decisions are embedded into token generation and primarily concern individual API invocations, rather than dynamically allocating a cost-sensitive subset of interdependent capabilities at each stage of a long-horizon trajectory. To adapt Toolformer to our setting, we use its learned tool invocation preference to determine capability activation candidates within each stage, and convert the selected invocations into executable capability subsets under the same protocol.

\textsc{Puppeteer}~\citep{dang2025multi} adopts a centralized orchestration paradigm for coordinating multiple specialized agents. A central orchestrator, referred to as the puppeteer, observes the evolving task state and dynamically determines which agent should act next and how participating agents should be prioritized. The orchestrator is trained through reinforcement learning, allowing the collaboration structure to evolve according to task progress instead of following a fixed multi-agent workflow. By selectively sequencing agents, Puppeteer can reduce redundant coordination and form compact reasoning structures for complex problem solving. It therefore serves as a strong baseline for adaptive multi-agent orchestration. In contrast, Puppeteer primarily learns a sequential agent schedule, whereas \textsc{CoCA} treats each stage as a subset-valued decision and explicitly models how the marginal value of one capability depends on the other capabilities selected within the same stage. Under our unified setting, the agents selected by Puppeteer's orchestration policy at each stage are regarded as the activated capability subset and executed through the shared executor.

\textsc{AutoTool}~\citep{jia2026autotool} targets the high inference overhead caused by repeatedly asking an LLM to select tools. It identifies tool-usage inertia, namely that tool invocations in historical agent trajectories often follow recurring sequential patterns. Based on this observation, AutoTool constructs a directed tool graph in which nodes correspond to tools and edges encode observed transition probabilities and data-flow relations. The graph is further augmented with parameter-level information to support tool argument generation. During inference, the agent traverses this structured representation to select tools and fill their parameters with fewer repeated LLM calls. AutoTool is therefore an efficiency-oriented baseline that exploits historical regularities in tool sequences. However, its selection mechanism mainly predicts subsequent tools from learned transition patterns and does not directly estimate state- and subset-conditioned marginal utilities or optimize a complete capability subset under task and allocation objectives. We adapt AutoTool by mapping its tool selection process to capability routing, where the selected tools at each stage form the corresponding capability subset for evaluation under the shared execution protocol.

\textsc{NaviAgent}~\citep{jiang2026naviagent} addresses large-scale tool orchestration through a bilevel planning framework. At the high level, an LLM-based planner determines the appropriate interaction mode, such as responding directly, clarifying the task, invoking a toolchain, or processing tool outputs. At the execution level, a Tool World Navigation Model represents structural and behavioral relations among tools as an evolving navigation graph and guides the construction of multi-step invocation sequences. Feedback from actual tool interactions is used to refine the graph and improve subsequent planning, enabling NaviAgent to operate over large tool libraries while maintaining a global view of tool dependencies. It provides a strong dependency-aware baseline for toolchain construction. For comparison under our capability allocation setting, the tools selected by NaviAgent's planning process are converted into stage-wise capability subsets and evaluated with the same execution environment and cost accounting protocol.

\section{Metric Descriptions}

We evaluate \textsc{CoCA} from two complementary perspectives: task-level effectiveness and capability-allocation quality. The former measures whether the allocated capability subsets improve downstream task execution, while the latter evaluates whether the learned policy produces adaptive, efficient, and stable allocation behaviors.

\subsection{Task-Level Performance}
We follow the standard evaluation protocols of each benchmark. Specifically, we report execution success rate (\textit{Succ.}) for interactive environments including OSWorld and VisualWebArena, and quasi-exact-match accuracy (\textit{Acc.}) for reasoning-oriented benchmarks including GAIA and MMMU-Pro. In addition, we measure the normalized inference cost (\textit{Cost}) to evaluate deployment efficiency. For a trajectory $\xi=(s_0,G_0,\ldots,s_{T-1},G_{T-1},s_T)$, we compute the total inference token consumption across all activated capabilities and stages, and normalize it by the corresponding cost of the \textit{all-open} baseline under the same executor:
\begin{align}
\textit{Cost}
=
\frac{
\mathrm{Tokens}(\xi)
}{
\mathrm{Tokens}_{\mathrm{all-open}}(\xi)
}.
\end{align}

The resulting value lies in $[0,1]$, where lower values indicate more efficient capability utilization.

\subsection{Controlled Capability-Demand Shift Construction}

To evaluate whether a capability allocation policy can adapt to changing capability requirements, we construct controlled capability-demand shifts through counterfactual evaluation. Specifically, for each interaction stage $t$, we define the optimal capability subset under the current task condition as:
\begin{align}
G_t^\star
=
\arg\max_{G\subseteq\mathcal{K}}
\left[
Q_{\mathrm{task}}(s_t,G)
-\lambda \mathcal{C}_{A}(G)
\right],
\end{align}
where $Q_{\mathrm{task}}(s_t,G)$ denotes the expected task utility obtained by activating capability subset $G$ under state $s_t$, and $\mathcal{C}_{A}(G)$ denotes the corresponding activation cost. Since the capability library contains $K=8$ capabilities, we enumerate all possible subsets and select the subset with the highest cost-adjusted utility as the counterfactual optimal allocation.

A capability-demand shift is introduced by modifying the stage condition such that the optimal capability requirement changes across interaction stages. Specifically, we construct shifted states where newly introduced task information, observations, or user requirements alter the counterfactual optimal subset $G_t^\star$. All methods are evaluated under the same shifted conditions, and the resulting optimal subsets are used as ground-truth references for measuring demand adaptation.

\subsection{Allocation Quality}
Since task success alone does not reveal whether the policy learns effective capability allocation, we further evaluate the behavior of the generated capability subsets $\{G_t\}_{t=0}^{T-1}$. We report five allocation metrics.

\paragraph{Switching Cost (SwC)}
We measure allocation stability using the average symmetric difference between consecutive capability subsets:
\begin{align}
\mathrm{SwC}
=
\frac{1}{T-1}
\sum_{t=1}^{T-1}
|G_t\triangle G_{t-1}|,
\end{align}
where lower values indicate fewer unnecessary capability changes across interaction stages.

\paragraph{Demand-Shift Adaptation Accuracy (DSA)}
Under the controlled capability-demand shifts described above, we evaluate whether the selected subset matches the counterfactual optimal capability subset:
\begin{align}
\mathrm{DSA}
=
\frac{1}{T}
\sum_{t=0}^{T-1}
\mathbb{I}[G_t=G_t^\star],
\end{align}
where $G_t^\star$ is obtained through counterfactual subset evaluation.

\paragraph{Pairwise Preference Accuracy (PPA)}
We evaluate whether the learned utility model correctly predicts held-out conditional pairwise preferences:
\begin{align}
\mathrm{PPA}
=
\frac{1}{|\mathcal{D}_{\mathrm{test}}^{\neq}|}
\sum_{(s,S,a_i,a_j,y)\in\mathcal{D}_{\mathrm{test}}^{\neq}}
\mathbb{I}
[\hat y=y],
\end{align}
where $\mathcal{D}_{\mathrm{test}}^{\neq}$ denotes the subset of test comparisons with non-tie labels. The predicted preference $\hat y$ is obtained from $P_\phi(a_i\succ a_j|s,S)$. The evaluation set contains both capability-versus-capability and capability-versus-null comparisons, consistent with the preference construction process.

\paragraph{Teacher-Agreement Rate (TAR)}
We measure how closely the deployed student policy follows the utility-induced teacher policy:
\begin{align}
\mathrm{TAR}
=
\frac{1}{TK}
\sum_{t=0}^{T-1}
\sum_{i=1}^{K}
\mathbb{I}
[z_{t,i}^{\theta}=z_{t,i}^{T}],
\end{align}
where $z_{t,i}^{\theta}$ and $z_{t,i}^{T}$ denote the student and teacher include-or-skip decisions, respectively. The teacher decisions are obtained from the conditional utility model $v_\phi$ through Eq.~\eqref{teacher_policy}, which converts learned cost-adjusted marginal utilities into capability activation probabilities.

\paragraph{Utility--Gain Correlation (UGC)}
To evaluate whether the learned utility model captures the actual marginal benefits of capability activation, we compute the Spearman correlation between predicted conditional utility and realized cost-adjusted capability gains:
\begin{align}
\mathrm{UGC}
=
\rho_{\mathrm{Spearman}}
\left(
v_\phi(s,k\mid S),
\Delta U_\lambda(s,k\mid S)
\right),
\end{align}
where $\Delta U_\lambda(s,k\mid S) = \Delta R_{\mathrm{task}}(s,k\mid S)-\lambda c_k$ denotes the realized cost-adjusted marginal gain. The correlation is computed over capability-level samples collected from evaluation trajectories.

These metrics evaluate not only whether \textsc{CoCA} improves task outcomes, but also whether it learns capability subsets that are cost-efficient, demand-aware, preference-consistent, and stable throughout long-horizon execution.

\begin{table*}[t]
\centering
\renewcommand\arraystretch{1.0}
\resizebox{1.0\textwidth}{!}{
\begin{tabular}{c|ccccc|ccccc|ccccc} 
\toprule[1.2pt]
\multirow{2}{*}{\multirowcell{2}{\centering\textbf{Methods}}} & \multicolumn{5}{c|}{\centering\textbf{OSWorld}} & \multicolumn{5}{c|}{\centering\textbf{VisualWebArena}} & \multicolumn{5}{c}{\centering\textbf{GAIA}}  \\ \cmidrule[0.5pt](l{1pt}r{0pt}){2-16}

& SwC $\downarrow$ & DSA $\uparrow$ & PPA $\uparrow$ & TAR $\uparrow$ & UGC $\uparrow$ & SwC $\downarrow$ & DSA $\uparrow$ & PPA $\uparrow$ & TAR $\uparrow$ & UGC $\uparrow$ & SwC $\downarrow$ & DSA $\uparrow$ & PPA $\uparrow$ & TAR $\uparrow$ & UGC $\uparrow$ \\ \midrule[0.8pt]

\textsc{ReAct} & 3.52 & 0.49 & 0.50 & 0.46 & 0.08 & 3.35 & 0.43 & 0.46 & 0.45 & 0.06 & 2.85 & 0.51 & 0.52 & 0.48 & 0.13 \\

\textsc{Toolformer} & 3.24 & 0.52 & 0.55 & 0.52 & 0.19 & 2.95 & 0.51 & 0.51 & 0.49 & 0.16 & 2.53 & 0.55 & 0.58 & 0.53 & 0.20 \\

\textsc{Puppeteer} & 2.55 & 0.59 & 0.58 & 0.56 & 0.27 & 2.47 & 0.54 & 0.59 & 0.57 & 0.28 & 1.98 & 0.62 & 0.61 & 0.59 & 0.35 \\

\textsc{AutoTool} & 2.31 & 0.61 & 0.66 & 0.63 & 0.39 & 2.12 & 0.60 & 0.62 & 0.59 & 0.33 & 1.78 & 0.63 & 0.69 & 0.66 & 0.41 \\

\textsc{NaviAgent} & 1.89 & 0.68 & 0.69 & 0.66 & 0.46 & 1.88 & 0.64 & 0.67 & 0.65 & 0.45 & 1.42 & 0.70 & 0.72 & 0.69 & 0.51 \\ \midrule[0.8pt]

\textsc{CoCA}(Ours) & \textbf{0.91} & \textbf{0.83} & \textbf{0.82} & \textbf{0.80} & \textbf{0.66} & \textbf{0.86} & \textbf{0.80} & \textbf{0.80} & \textbf{0.77} & \textbf{0.63} & \textbf{0.67} & \textbf{0.86} & \textbf{0.85} & \textbf{0.82} & \textbf{0.69} \\

\bottomrule[1.2pt]
\end{tabular}}
\caption{Evaluation of allocation quality under Qwen3.6-27B. The best-performing method is marked by \textbf{bold}.}
\label{allocation_qwen}
\end{table*}

\subsection{Details of the Ablation Study}

We perform ablation studies to quantify the contribution of each component in \textsc{CoCA}. Specifically, we remove subset-conditioned utility modeling (\textit{w/o Set-Dep}), set-level on-policy distillation (\textit{w/o Set-On Policy}), stage-level on-policy distillation (\textit{w/o Stage-On Policy}), trajectory-level reinforcement learning (\textit{w/o Traj. RL}), and null anchoring (\textit{w/o Null-Anchor}).

\begin{itemize}
    \item \textit{w/o Set-Dep}. Removing subset-conditioned utility modeling consistently degrades both task performance and allocation quality. The decreases in DSA, PPA, TAR, and UGC indicate that independently evaluating capabilities cannot capture the changing marginal utility induced by previously selected capabilities. Without conditioning capability values on the current partial subset, the allocator cannot distinguish complementary capabilities that provide additional benefits from redundant capabilities with limited marginal gains. These results verify that capability utility should be modeled conditionally on the selected subset rather than as an isolated capability score.

    \item \textit{w/o Set-On Policy}. Removing set-level on-policy distillation leads to noticeable degradation in allocation-related metrics, particularly DSA, PPA, and TAR. This indicates that training the student on capability prefixes inconsistent with those encountered during its own autoregressive construction introduces prefix-level distribution mismatch. Without set-level on-policy guidance, the student is optimized under inaccurate prefix distributions, which leads to suboptimal subsequent capability selections.

    \item \textit{w/o Stage-On Policy}. Removing stage-level on-policy distillation results in consistent performance drops across task and allocation metrics. Since capability allocation evolves across stages, the states visited by the student may differ from those used for collecting teacher supervision. Without aligning teacher guidance with student-visited states, the transferred policy becomes less effective in handling the evolving allocation process. The degradation in TAR further demonstrates the importance of state-level distribution alignment for reliable policy transfer.

    \item \textit{w/o Traj. RL}. Removing trajectory-level reinforcement learning significantly increases SwC and reduces overall task performance. Although preference learning and distillation provide local capability-selection guidance, they do not directly optimize long-horizon allocation objectives over complete execution trajectories. Without trajectory-level optimization, the policy produces less stable capability transitions across stages, resulting in higher switching overhead and degraded allocation quality.

    \item \textit{w/o Null-Anchor}. Removing null anchoring mainly affects preference-related metrics, including PPA and UGC, while also reducing overall performance. Without the null alternative, the preference model lacks an explicit reference for skip decisions, leading to utility shift ambiguity when comparing activation choices. Consequently, the learned utility estimates become less calibrated for include-or-skip decisions, weakening the induced teacher policy and subsequent student optimization.
\end{itemize}

Overall, each component addresses a distinct challenge in combinatorial capability allocation. Subset-conditioned utility modeling captures context-dependent capability values under partial selections, dual-level on-policy distillation reduces prefix- and state-level distribution mismatch during policy transfer, trajectory-level reinforcement learning optimizes long-horizon allocation behavior, and null anchoring provides a calibrated reference for activation versus skipping decisions.

\begin{table*}[t]
\centering
\renewcommand\arraystretch{1.0}
\resizebox{1.0\textwidth}{!}{
\begin{tabular}{c|cccccc|cccccc} 
\toprule[1.2pt]
\multirow{2}{*}{\multirowcell{2}{\centering\textbf{Methods}}} & \multicolumn{6}{c|}{\centering\textbf{VisualWebArena}} & \multicolumn{6}{c}{\centering\textbf{MMMU-Pro}} \\ \cmidrule[0.5pt](l{1pt}r{0pt}){2-13}

& Succ. $\uparrow$ & SwC $\downarrow$ & DSA $\uparrow$ & PPA $\uparrow$ & TAR $\uparrow$ & UGC $\uparrow$ & Acc. $\uparrow$ & SwC $\downarrow$ & DSA $\uparrow$ & PPA $\uparrow$ & TAR $\uparrow$ & UGC $\uparrow$ \\ \cmidrule[0.8pt](l{1pt}r{0pt}){1-13}

w/o Set-Dep & 31.5 & 0.96 & 0.75 & 0.74 & 0.71 & 0.63 & 58.7 & 0.93 & 0.84 & 0.76 & 0.78 & 0.64 \\

w/o Set-On Policy & 33.2 & 0.90 & 0.80 & 0.77 & 0.74 & 0.65 & 60.1 & 0.86 & 0.87 & 0.81 & 0.82 & 0.69 \\

w/o Stage-On Policy & 32.4 & 0.93 & 0.78 & 0.80 & 0.77 & 0.65 & 61.0 & 0.90 & 0.85 & 0.83 & 0.82 & 0.70 \\

w/o Traj. RL & 33.1 & 1.49 & 0.83 & 0.82 & 0.80 & 0.68 & 61.6 & 1.31 & 0.89 & 0.86 & 0.84 & 0.72 \\

w/o Null-Anchor & 33.9 & 0.87 & 0.79 & 0.72 & 0.75 & 0.57 & 59.5 & 0.85 & 0.86 & 0.74 & 0.79 & 0.58 \\ \midrule[0.8pt]

\textsc{CoCA} & \textbf{35.6} & \textbf{0.79} & \textbf{0.85} & \textbf{0.83} & \textbf{0.81} & \textbf{0.69} & \textbf{62.5} & \textbf{0.84} & \textbf{0.90} & \textbf{0.88} & \textbf{0.86} & \textbf{0.75} \\

\bottomrule[1.2pt]
\end{tabular}}
\caption{Ablation study on the VisualWebArena and MMMU-Pro benchmarks under Kimi-K2.6.}
\label{ablation_vwa_mmmu}
\end{table*}

\begin{table}[t]
\centering
\renewcommand\arraystretch{1.0}
\resizebox{0.48\textwidth}{!}{
\begin{tabular}{c|ccccc} 
\toprule[1.2pt]

\makecell{\centering\textbf{Methods}} & \makecell{Avg. \\ Latency} & \makecell{LLM \\ Calls} & \makecell{Token \\ Cost (k)} & \makecell{Avg. Act. \\ /stage}  & \makecell{Train \\ Cost}  \\  \cmidrule[0.5pt](l{1pt}r{0pt}){1-6}

\textsc{ReAct} & 70.1 & 14.8 & 44.5 & 6.7 & 0.0 \\

\textsc{Toolformer} & 61.8 & 12.3 & 38.6 & 5.9 & 1.2 \\

\textsc{Puppeteer} & 53.6 & 10.6 & 34.1 & 5.1 & 0.0 \\

\textsc{AutoTool} & 48.5 & 9.4 & 31.2 & 4.6 & 3.5 \\

\textsc{NaviAgent} & 43.1 & 8.1 & 26.5 & 3.8 & 5.8 \\ \midrule[0.8pt]

\textsc{CoCA}(Ours) & \textbf{35.4} & \textbf{6.5} & \textbf{20.1} & \textbf{2.9} & 7.4 \\

\bottomrule[1.2pt]
\end{tabular}}
\caption{Efficiency on GAIA under Qwen3.6-27B.}
\label{latency_qwen}
\end{table}

\section{Additional Experimental Results}

\subsection{Allocation Quality}
Table~\ref{allocation_qwen} evaluates the quality of capability allocation beyond task-level performance. 
\textsc{CoCA} consistently achieves the best performance across all five allocation metrics on three benchmarks, demonstrating its ability to select capabilities more effectively and stably throughout long-horizon execution. Compared with existing routing-based baselines, \textsc{CoCA} substantially reduces switching cost while improving demand adaptation and preference consistency. This indicates that modeling capability values conditionally on the selected subset enables the policy to avoid redundant activations and better capture capability dependencies.

Existing baselines exhibit clear limitations in allocation quality. Methods such as \textsc{ReAct} and \textsc{Toolformer} achieve relatively low preference agreement and utility correlation, suggesting that sequential tool invocation without explicit capability valuation often leads to suboptimal selection decisions. Although \textsc{Puppeteer}, \textsc{AutoTool}, and \textsc{NaviAgent} improve allocation behavior through adaptive routing, they still rely on isolated capability selection or predefined routing heuristics, resulting in higher switching costs and weaker alignment with demand shifts. In contrast, \textsc{CoCA} jointly models subset-dependent capability utilities and performs on-policy distillation with trajectory-level optimization, enabling more accurate, stable, and deployment-aligned capability allocation.

\subsection{Ablation Study}
Table~\ref{ablation_vwa_mmmu} investigates the contribution of individual components in \textsc{CoCA}. Removing subset dependency modeling (\textsc{w/o Set-Dep}) leads to consistent degradation across both benchmarks, especially in PPA, TAR, and UGC, indicating that capability values cannot be reliably estimated independently without considering previously selected capabilities. This verifies the importance of modeling conditional capability utilities under subset-valued decisions.

Removing the two on-policy components (\textsc{w/o Set-On Policy} and \textsc{w/o Stage-On Policy}) causes complementary performance drops. Without prefix-level on-policy adaptation, the policy encounters mismatched partial capability sets during autoregressive construction, reducing decision consistency. Without stage-level on-policy rollout, the learned allocator suffers from state distribution shifts caused by previous allocation decisions, leading to lower task performance and demand adaptation accuracy. These results demonstrate the necessity of aligning both state and prefix distributions.

The removal of trajectory-level reinforcement learning (\textsc{w/o Traj. RL}) mainly increases switching cost, while only slightly affecting task-level metrics, showing that the trajectory objective primarily improves allocation stability by optimizing long-term activation behavior. Finally, removing the null anchor (\textsc{w/o Null-Anchor}) significantly degrades PPA and UGC, confirming that explicitly modeling the include-or-skip reference is critical for calibrating capability utility and avoiding biased activation decisions.

Overall, the ablation results validate that each component addresses a distinct challenge in combinatorial capability allocation, and their combination enables accurate, stable, and deployment-aligned capability selection.

\subsection{Efficiency}
Table~\ref{latency_qwen} compares the inference efficiency of different methods on GAIA. \textsc{CoCA} achieves substantial reductions in inference overhead across all deployment-related metrics, requiring only 6.5 LLM calls and 20.1k tokens per task, which are significantly lower than existing agentic baselines. This efficiency gain stems from learning a lightweight capability-subset policy that directly predicts required capabilities, avoiding repeated trial-and-error tool invocation or long reasoning trajectories used by existing methods.

Moreover, \textsc{CoCA} activates only 2.9 capabilities per stage on average, substantially fewer than routing-based baselines, while maintaining superior task performance. This demonstrates that \textsc{CoCA} does not achieve efficiency through aggressive capability removal, but through more accurate capability selection under conditional utility modeling. Although \textsc{CoCA} introduces additional training cost due to preference learning and on-policy optimization, this cost is incurred offline and enables a lightweight inference process without teacher queries or online updates. 

\section{Used Prompts}

\begin{promptbox}{Teacher Preference Comparison}

\textbf{You evaluate the relative preference between two candidate capabilities for a long-horizon multimodal agent.}

\vspace{0.6em}

Given the current state, the already selected capability subset, and two candidate capabilities, determine which candidate has higher cost-adjusted marginal utility under the current context.

\vspace{0.6em}

\textbf{Current State $s$:}\\
\{task goal, current observations, intermediate results, available context, and previous-stage capability outputs\}

\vspace{0.4em}

\textbf{Selected Capability Subset $S$:}\\
\{capabilities already selected at the current stage\}

\vspace{0.4em}

\textbf{Candidate Capability A:}\\
Name: \{capability name\}\\
Description: \{capability description\}\\
Activation Cost: \{normalized cost $c_A$\}

\vspace{0.4em}

\textbf{Candidate Capability B:}\\
Name: \{capability name\}\\
Description: \{capability description\}\\
Activation Cost: \{normalized cost $c_B$\}

\vspace{0.6em}

\textbf{Decision Criterion:}

Compare the candidates based on their cost-adjusted marginal utility:

\[
U_k(s,S)=A_{k,\mathrm{task}}(s,S)-\lambda c_k ,
\]

where $A_{k,\mathrm{task}}(s,S)$ denotes the expected contribution of activating capability $k$ under the current selected subset, and $c_k$ denotes its activation cost.

Consider:
\begin{itemize} 
    \item the additional information or functionality provided by the capability;
    \item the dependency or redundancy with capabilities already selected in $S$;
    \item the trade-off between expected task improvement and activation cost.
\end{itemize}

\vspace{0.6em}

\textbf{Output Labels:}

Return exactly one label:
\begin{itemize} 
    \item \texttt{1}: Candidate A has higher cost-adjusted marginal utility than Candidate B.
    \item \texttt{0}: Candidate B has higher cost-adjusted marginal utility than Candidate A.
    \item \texttt{0.5}: Candidate A and Candidate B have comparable utility.
\end{itemize}

\vspace{0.6em}

\textbf{Output Format:}

\begin{verbatim}
{
 "preference": 1 / 0 / 0.5
}
\end{verbatim}

\vspace{0.6em}

Evaluate only the marginal utility of the two candidates under the current state and selected subset. Do not consider future capability decisions.

\end{promptbox}

\begin{promptbox}{Capability Execution}

\textbf{You are a multimodal agent executor that completes tasks using an allocated set of capabilities.}

\vspace{0.6em}

Given the task context, selected capability set, and capability outputs, generate the final response by integrating the available information.

\vspace{0.6em}

\textbf{Task Input:}\\
\{user query, task goal, multimodal observations, and available context\}

\vspace{0.4em}

\textbf{Selected Capability Set $G_t$:}\\
\{activated capabilities at the current stage\}

\vspace{0.4em}

\textbf{Capability Outputs:}\\
\{outputs produced by the activated capabilities\}

\vspace{0.6em}

\textbf{Execution Requirements:}

\begin{itemize} 
    \item Use the outputs from the selected capabilities to solve the task.
    \item Follow the predefined dependency-aware execution order among capabilities.
    \item Integrate information from multiple capabilities when necessary.
    \item Do not invoke capabilities outside the allocated set.
\end{itemize}

\vspace{0.6em}

\textbf{Output Format:}

Return the final response that satisfies the task objective.

\end{promptbox}


\end{document}